\def\year{2022}\relax
\documentclass[letterpaper]{article} 
\usepackage{aaai22}  
\usepackage{times}  
\usepackage{helvet}  
\usepackage{courier}  
\usepackage[hyphens]{url}  
\usepackage{graphicx} 
\usepackage{natbib}  
\usepackage{caption} 
\DeclareCaptionStyle{ruled}{labelfont=normalfont,labelsep=colon,strut=off} 
\usepackage{algorithm}
\usepackage{algorithmic}

\usepackage{amsmath,amssymb,amsthm} 
\usepackage{xcolor}
\usepackage{enumitem}
\usepackage{hyperref}
\usepackage{tikz}
\usepackage{pgfplots}
\usetikzlibrary{calc, shapes.geometric, shapes.misc, positioning, intersections, pgfplots.fillbetween, automata}
\pgfplotsset{compat=1.15}
\usepackage{xspace}

\newcommand{\R}{\mathcal{R}}
\newcommand{\N}{\mathcal{N}}
\newcommand{\W}{\mathcal{W}}
\newcommand{\X}{\mathcal{X}}
\newcommand{\A}{\mathcal{A}}
\newcommand{\G}{\mathcal{G}}
\newcommand{\F}{\mathcal{F}}
\newcommand{\bbR}{\mathbb{R}}

\newcommand{\one}{\textbf{1}}
\newcommand{\norm}[1]{\left\lVert #1 \right\rVert}
\newcommand{\abs}[1]{\left| #1 \right|}
\newcommand{\mso}{\{\!\!\{}
\newcommand{\msc}{\}\!\!\}}
\newcommand{\scalprod}[2]{\langle #1 , #2 \rangle}

\newtheorem{definition}{Definition}
\newtheorem{theorem}{Theorem}
\newtheorem{proposition}{Proposition}
\newtheorem{lemma}{Lemma}
\newtheorem{corollary}{Corollary}

\graphicspath {{figures/}}

\usepackage{newfloat}
\usepackage{listings}
\floatstyle{ruled}
\newfloat{listing}{tb}{lst}{}
\floatname{listing}{Listing}
\title{Graph Filtration Kernels}
\author{
	Till H.~Schulz\textsuperscript{\rm 1}, Pascal Welke\textsuperscript{\rm 1}, Stefan Wrobel\textsuperscript{\rm 1,2}
}
\affiliations{
    \textsuperscript{\rm 1} Dept. of Computer Science, Universit{\"a}t Bonn, Germany\\
	\textsuperscript{\rm 2} Center for Machine Learning, Sankt Augustin, Germany\\
}
\usepackage{bibentry}

\begin{document}

	\maketitle
	
	\begin{abstract}

The majority of popular graph kernels is based on the concept of Haussler's $\mathcal{R}$-convolution kernel and defines graph similarities in terms of mutual substructures. 
In this work, we enrich these similarity measures by considering graph filtrations:
Using meaningful orders on the set of edges, which allow to construct a sequence of nested graphs, we can consider a graph at \emph{multiple granularities}. 
For one thing, this provides access to features on different levels of resolution. 
Furthermore, rather than to simply compare frequencies of features in  graphs, it allows for their comparison in terms of \emph{when} and for \emph{how long} they exist in the sequences. 
In this work, we propose a family of graph kernels that incorporate these existence intervals of features. 
While our approach can be applied to arbitrary graph features, we particularly highlight Weisfeiler-Lehman vertex labels, leading to efficient kernels. 
We show that using Weisfeiler-Lehman labels over certain filtrations strictly increases the expressive power over the ordinary Weisfeiler-Lehman procedure in terms of deciding graph isomorphism.
In fact, this result directly yields more powerful graph kernels based on such features and has implications to graph neural networks due to their close relationship to the Weisfeiler-Lehman method.
We empirically validate the expressive power of our graph kernels and show significant improvements over state-of-the-art graph kernels in terms of predictive performance on various real-world benchmark datasets. 
	\end{abstract}


\section{Introduction}
    Graph-structured data is prevalent in countless domains such as social networks, biological interaction graphs and molecules. 
	A central task on this kind of data is the classification of graphs. 
	Perhaps the most established machine learning methods for graph classification are graph kernels which, even in the advent of neural network approaches, remain highly relevant due to their remarkable predictive performance \citep{DBLP:journals/ans/KriegeJM20}. 
	The majority of graph kernels are instances of Haussler's $\R$-convolution kernel \citep{Haussler99} which define graph similarity in terms of pairwise similarities between the graphs' substructures. 
	Some well-known representatives are the Weisfeiler-Lehman graph kernels \citep{DBLP:journals/jmlr/ShervashidzeSLMB11}, the shortest-path kernel \citep{DBLP:conf/icdm/BorgwardtK05} and the cyclic pattern kernel \citep{DBLP:conf/kdd/HorvathGW04}.
	Generally, such kernels revert to simply comparing substructures in terms of equivalence. We refer to these kernels as \emph{histogram kernels}. 
	While they prove to be successful on many classification tasks, this notion of equivalence is often too rigid in terms of similarity functions.
	
	Motivated by this limitation, we introduce a family of graph kernels which regard a graph at multiple levels of resolution.
	This is realized using the concept of graph filtrations, which define sequences of nested subgraphs that differ only in the sets of edges. 
	Such a sequence can be viewed as incremental refinements that construct a graph by gradually adding sets of edges. 
	An example depicting this concept is found in Fig. \ref{fig:pipeline}.
	Clearly, with changing sets of edges, the graph features change as well. 
	That is, features occurring at some point in the sequence may disappear at a later moment.  
	In this work, we track such feature existence intervals which allows for a comparison not only in terms of feature frequency but also by \emph{when} and for \emph{how long} they exist. 
	This comparison of feature occurrence distributions is realized using the Wasserstein distance which, using recent results, allows for proper kernel functions on this kind of information.     
	A benefit of this approach is the kernels' ability to handle continuous edge attributes. 

	Our main contributions are summarized as follows:
	\begin{enumerate}
		\item We introduce a general graph kernel framework which defines similarity by comparing graph feature occurrence distributions over sequences of graph resolutions and show that such kernels generalize histogram kernels.
		\item We particularly consider the well-known Weisfeiler-Lehman subtree features and show that there exist filtration kernels using such features which \emph{strictly increase the expressive power} over the ordinary Weisfeiler-Lehman subtree kernel. 
		\item We empirically validate our theoretical findings on the expressive power of our kernels and furthermore provide experiments on real-world benchmark datasets which show a favorable performance of our approach compared to state-of-the-art graph kernels. 
		
	\end{enumerate}
	

 \section{Background}
 \label{sec:background}
	 	
	 \paragraph{Graph kernels.}
		\emph{Kernels} are functions of the form $k: \X \times \X \rightarrow \bbR$ which define similarity measures on non-empty sets $\X$.
		More precisely, $k$ is a kernel if there exists a mapping $\varphi: \X \rightarrow \mathcal{H}_k$ with $k(x,y) = \langle \varphi(x), \varphi(y) \rangle$ where $\langle \cdot, \cdot \rangle$ is the inner product in the associated Hilbert space $\mathcal{H}_k$. 
		With Haussler's work on \emph{convolution kernels} over discrete structures \citep{Haussler99}, kernel methods became widely applicable on graphs. 
		The concept of $\R$-convolution kernels provides a general framework which can be used to construct graph kernels by defining graph similarity in terms of their aggregated substructure similarities. 
		More formally, let there be a function decomposing the graph $G$ into the set of its substructures $\X_G$. 
		Then, for graphs $G,G'$, the (simplified) $\R$-convolution kernel is defined by 
		\begin{equation}
			k(G,G') = \sum_{(x,y) \in \X_G \times \X_{G'}} \kappa(x,y)
		\end{equation}
		where $\kappa(x,y)$ is a kernel on the substructures. 
		To guarantee convergence, we assume $\X_G$ to be finite. 
		As this decomposition can be perceived as a feature extraction process, we will commonly refer to such substructures as \emph{features}. 
		In a majority of graph kernels the function $\kappa$ is simply the Dirac delta which amounts to $1$ if the features $x$ and $y$ are equivalent and $0$ otherwise.\footnote{
		    We note that kernels generally allow an individual weighting of features. 
		    However for reasons of simplicity, we omit this aspect in the further discussions.
		}
		Thus, $k(G,G')$ essentially measures graph similarity by counting pairs of equivalent features.
		Some well-known examples are the Weisfeiler-Lehman subtree kernel \citep{DBLP:journals/jmlr/ShervashidzeSLMB11} and cyclic pattern kernels \citep{DBLP:conf/kdd/HorvathGW04}.
		Such kernels can alternatively be described by the inner product of explicit feature vectors
		\begin{equation}
		    \label{eq:feature_vector}
			\varphi(G) = [c(f_1(G)), c(f_2(G)), \ldots]
		\end{equation}
		where $c(f_i(G))$ indicates the counts of features $f_i \in \F$ in graph $G$ over some fixed feature domain $\F$. 
		We refer to this kind of kernels as \emph{histogram} kernels.  
		Note that histogram kernels can equivalently be expressed as a sum of feature frequency products, i.e., 
		\begin{equation}
		    \label{eq:hist_kernel}
		    K^\F_H(G,G) = \sum_{f \in \F} c(f(G)) ~ c(f(G')) ~~.
		\end{equation}
		
	\paragraph{Wasserstein distance.}
	 	The Wasserstein distance is a distance function between probability distributions based on the concept of optimal mass transportation.
	 	In the following, we specifically consider the $1$-Wasserstein distance for discrete distributions, i.e., histograms, and refer to it as the Wasserstein distance. 
		For more general definitions see e.g. \citet{compOT}.
	 	Intuitively, the Wasserstein distance can be viewed as the minimum cost necessary to transform one pile of earth into another.
	 	It is, therefore, also known as the \emph{earth movers} distance or \emph{optimal transportation} distance.  
	 	More formally, given two vectors $x \in \bbR^n$ and $x' \in \mathbb{R}^{n'}$ with $\norm{x}_1=\norm{x'}_1$ and a cost matrix $C_d^{n \times n'}$ containing pairwise distances between entries of $x$ and $x'$, the \emph{Wasserstein distance}
	 	is defined by 
	 	\begin{equation}\label{eq:wasserstein}
	 		\W_d(x,x') = \min_{T \in \mathcal{T}(x,x')} \langle T,C_d \rangle
	 	\end{equation} 
	 	with $\mathcal{T}(x,x') \subseteq \mathbb{R}^{n \times n'}$ and $T\one_{n'}=x$, $\one_n^\top T=x'$ for all $T \in \mathcal{T}(x,x')$, where $\langle \cdot,\cdot \rangle$ is the Frobenius inner product.
	 	The function $d$ defining the costs in $C_d$ between entries of $x$ and $x'$ is called \emph{ground distance}.
	 	If the ground distance is a metric, then the Wasserstein distance is a metric \citep[][Sec. 2.4]{compOT}.
	

\section{From Filtrations to Distances}
	We now outline the concept of tracking graph features over sequences of different resolutions and define distance measures between graphs using this kind of information. 
	
	\subsection{Feature Persistence}
		The general idea of \emph{feature persistence} in graphs is derived from persistent homology, which refers to a method in computational topology that aims at measuring topological features at various levels of resolution \citep{DBLP:books/daglib/0025666}. 
		Persistent homology is applied in a wide range of topological data analysis tasks and has recently become a popular tool for analyzing topological properties in graphs \citep{DBLP:journals/ans/AktasAF19pershomsurvey}. 
		In this work, we adopt some of its basic concepts and specifically fit them to describe the idea of feature persistence. 
		
		Intuitively speaking, feature persistence tracks the lifespans of graph features in evolving graphs. 
		That is, it records the intervals during which occurrences of a specific feature appear in sequentially constructed graphs.
		Such a graph sequence is defined by a graph filtration which is essentially an ordered graph refinement that constructs a graph by gradually adding sets of edges. 
		More precisely, given a graph $G=(V,E)$, a \emph{graph filtration} $\A(G)$ is a sequence of graphs 
		\begin{equation} 
			\label{eq:filtration} 
			G_1 \subseteq G_2 \subseteq \ldots \subseteq G_k = G \ ,
		\end{equation}
		where $\subseteq$ denotes the subgraph relation and $G_i=(V,E_i \subseteq E)$ is called a \emph{filtration graph} of $G$. 
		Hence, filtration graphs differ only in the sets of edges and describe a sequence in which the last element is the graph $G$ itself. 
		Without loss of generality, we assume $G$ to be edge-weighted by a function $w: E(G) \to \mathbb{R}_+$ such that the filtration $\A(G)$ is implicated by a sequence of decreasing values
		\begin{equation} 
			\label{eq:weights} 
			\alpha_1 \geq \alpha_2 \geq \ldots \geq \alpha_k = 0 \ ,
		\end{equation}
		where a filtration graph $G_i$ is induced by the set of edges with weights greater or equal $\alpha_i$, i.e., $G_i = (V,\{e \in E:~ w(e) \geq \alpha_i\})$. 
		Thus, $\A$ is determined by $\{ \alpha_1, \ldots, \alpha_k \}$ 
		and isomorphic graphs $G,G'$ (considering edge weights) generate equivalent filtrations $\A(G) \equiv \A(G')$.
 
		While traditional persistent homology tracks lifespans of topological features such as connected components and cycles, our notion of feature persistence is concerned with arbitrary graph features. 
		More precisely, given some feature of interest $f$ (e.g., Weisfeiler-Lehman vertex colors, c.f. Fig. \ref{fig:pipeline}) and a graph filtration $\A(G)$, feature persistence describes the set of occurrence intervals of $f$ over the sequence $\A(G)$.
		This concept can very intuitively be depicted using (discrete) persistence barcodes shown in Fig.~\ref{fig:pipeline}.
		Each bar in the barcode diagram corresponds to the lifespan of a feature occurrence. 
		Furthermore, considering a graph at different levels of resolutions provides access to features not necessarily present in the graph $G$ itself. 
		\begin{figure*}[t]
			\centering
			\resizebox{1\textwidth}{!}{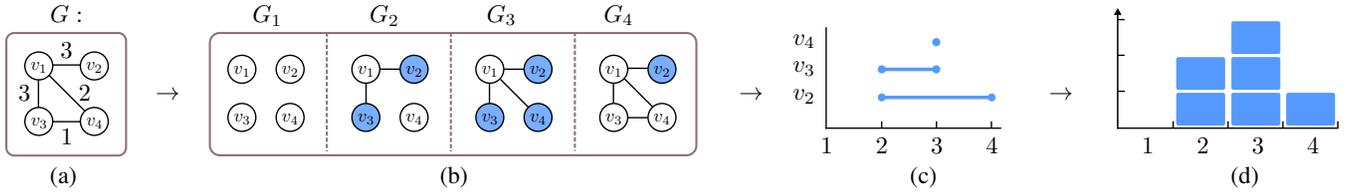}
			\caption{Consider the simple Weisfeiler-Lehman label $f$ corresponding to a vertex having exactly one neighbor. Each occurrence of this feature is individually marked in the filtration graphs $\A(G)$ shown in (b). The barcode in (c) depicts the existence intervals of each such feature occurrence. This information is then aggregated into a filtration histogram $\phi^\A_f(G)$ in (d). }
			\label{fig:pipeline}
		\end{figure*}
	
	\subsection{Wasserstein Distance on Filtration Histograms}
		The majority of traditional graph kernels defines the similarity measure in terms of the number of mutual substructures. 
		Often this comes down to simply comparing frequencies of features. 
		Using the concept of feature persistence, we are able to define much finer similarity measures on graphs. 
		We achieve this by defining a distance function on histograms which aggregate lifespans of feature occurrences.
		This distance measure compares graphs in terms of \emph{when} and \emph{for how long} a certain feature appears in the filtration instead of merely considering how often it occurs. 
		The underlying intuition is that features occurring close to each other in the filtration sequence indicate a higher similarity than those lying farther apart. 
		A natural choice for this distance function is the Wasserstein distance as the aggregated feature occurrence lifespans directly translate into $1$-dimensional distributions (see Fig.~\ref{fig:pipeline}).
		
		In order to define a distance measure w.r.t. \emph{a single} feature $f$ on graphs $G,G'$, we aggregate the feature persistence information of $G$ and $G'$  into a single histogram.
		This process is visualized in Fig.~\ref{fig:pipeline}. 
		Such a histogram essentially accumulates all feature lifespans of a particular feature and hence reflects the number of feature occurrences in each filtration graph.\footnote{
		    Note that aggregating all feature occurrence intervals into a single histogram clearly loses information on the distribution of the individual lifespans.  
		    However, while a pairwise comparison of persistence intervals has been shown to lead to valid kernels in the context of persistent homology \citep{DBLP:conf/cvpr/ReininghausHBK15}, our approach relies on a single histogram representation, which we show leads to very powerful kernel functions, nevertheless.
		}
		\begin{definition}[Filtration Histogram]
		    \label{def:filtration_histogram}
			Given a graph $G$ together with a length-$k$ filtration $\A(G)$ and a feature $f$, the function $\phi^\A_f: G \rightarrow \bbR^k$ maps $G$ to its \emph{filtration histogram} which counts occurrences of $f$ in each filtration graph of $G$.
		\end{definition}
		
		In the remainder of this work, we often omit the filtration function $\A$ in the notations if it is either irrelevant or clear from the context.
		
		Using filtration histograms allows for the application of natural distance measures such as the Wasserstein distance. 
		Intuitively speaking, the Wasserstein distance between such histograms describes the cost of shifting (accumulated) feature lifespans into another. 
		\begin{definition}[Filtration Histogram Distance]
			Given graphs $G,G'$, a filtration histogram mapping $\phi^\A_f: \G \rightarrow \bbR^k$ together with a distance function $d: \A_\alpha \times \A_\alpha \rightarrow \bbR$ on associated values $\A_\alpha = \{ \alpha_1,\ldots,\alpha_k \}$, the \emph{filtration histogram distance} is given by 
			\begin{equation}
				\label{eq:histogram_distance}
				\W_d(\phi^\A_f(G),\phi^\A_f(G')) ~~.
			\end{equation}
		\end{definition}
		The ground distance $d$ of the Wasserstein distance defines how feature occurrences at different points in the filtrations are being compared to each other.
		Since values in $\A_\alpha$ can be viewed as points on the timeline $[\alpha_1,\alpha_k]$, a natural choice for this distance is the simple euclidean metric 
		\begin{equation}
			\label{eq:grounddistance_1dim}
			d^1(\alpha_i,\alpha_j) = |\alpha_i-\alpha_j|
		\end{equation}
		on elements $\alpha_i,\alpha_j \in \A_\alpha$. 
		Furthermore, while the Wasserstein distance has cubic time complexity in the length of filtrations in general, this reduces to a \emph{linear} time complexity when employing $d^1$ on the real line as ground distance \citep[][Rem. 2.30]{compOT}. %
	

\section{Graph Filtration Kernels}
    \label{sec:filtration_kernels}
	Recent results in optimal transport theory give rise to proper kernel functions using the above filtration histogram distances when equipped with a suitable ground distance function \citep{DBLP:conf/nips/LeYFC19}. 
	In fact, it can be shown that utilizing the euclidean ground distance $d^1$ yields positive semi-definite kernels. 
	These kernels serve as building blocks for our final filtration graph kernels. 
	More precisely, we construct graph kernels by combining multiple base kernels $k_f$ over a set of features $f \in \F$.
	Each such base kernel is concerned with a single feature and defines a similarity between graphs $G$ and $G'$ w.r.t. this particular feature. 
	While the number of features may potentially be infinite, we can show that it suffices to consider only such features which appear in \emph{both} graphs. 
	
	Since filtration histograms (Def. \ref{def:filtration_histogram}) are not necessarily of equivalent mass, a normalization is necessary to fit the requirement of general Wasserstein distances. 
    We denote such mass-normalized filtration histograms using function $\hat{\phi}_f$.
	The base kernel on graphs $G,G'$ w.r.t. feature $f$ (and parameter $\gamma \in \mathbb{R}_+$) is then defined over normalized histograms as follows
	\begin{equation}
		k^\A_f(G,G') = e^{-\gamma \W_{d^1}(\hat{\phi}^\A_f(G),\hat{\phi}^\A_f(G'))} ~~.
	\end{equation}
	

	    Filtration kernels can be constructed in form of linear combinations of base kernels.
		That is, they are sums of kernels $k_f$ over features $f \in \F$. 
		Note that the mass-normalization of filtration histograms results in a loss of information on the number of occurrences of features.
	    However, this frequency information is often quite crucial. 
	    By introducing weights corresponding to the original masses of histograms, this information loss can be in part reverted. 
	    That is, we weigh each base kernel $k_f(G,G')$ using the original histogram masses of $\phi_f(G)$ and $\phi_f(G')$, i.e., $\norm{\phi_f(G)}_1$ resp. $\norm{\phi_f(G')}_1$.
		\begin{definition}[Filtration Kernel]
			Given graphs $G,G'$, a filtration function $\A$, and a set of features $\F$ the \emph{filtration kernel} is given by 
			\begin{equation}
			    \label{eq:filtration_kernel}
				K^{\F,\A}_{\text{Filt}}(G,G')=\sum_{f \in \F} k^\A_f(G,G') ~ \norm{\phi^\A_f(G)}_1 \norm{\phi^\A_f(G')}_1
			\end{equation}
		\end{definition}
		Notice the special case that a feature $f$ does not appear in the filtration sequence of $G$.
		Hence, the corresponding histogram has zero mass and the filtration histogram distance to some non-zero mass histogram is not properly defined. 
		This issue can formally be solved by introducing a dummy histogram entry as discussed in App. A. 
		We note that since the kernel $k_f(G,G')$ is multiplied by the histogram mass $\norm{\phi_f(G)}_1$, which in this particular case amounts to $0$, this is only a formal issue and becomes non-relevant in the kernel computation.
		From a similar argument it follows that only those features $f \in \F$ positively contribute to the similarity measure which appear in both graphs $G$ and $G'$.
		The proof of the following theorem can be found in App. A. 
		\begin{theorem}
		    \label{thm:psd}
			The Filtration Kernel $K^\F_{\text{Filt}}$ is positive semi-definite.
		\end{theorem}
		
		The filtration kernel is closely related to the histogram kernel (c.f. Eq. \ref{eq:hist_kernel}). 
		In fact, the histogram kernel is a special case of $K^\F_{\text{Filt}}$. 
		\begin{proposition}
		    \label{prop:filtration_to_histogram_kernel}
			For filtrations of length $k=1$ (i.e., filtrations consist of only the graph itself), the filtration kernel $K^\F_{\text{Filt}}$ reduces to the histogram kernel $K^\F_H$.
		\end{proposition}
		
		Filtration kernels can alternatively be defined as \emph{products} over base kernels. 
		In App. B
		, we consider such a variant which for $k=1$ reduces to the radial basis function (RBF) kernel. 

	    


\section{A Kernel Instance Using Weisfeiler-Lehman Subtree Features}
	We now discuss a concrete instance of the filtration kernel family which uses the well-known Weisfeiler-Lehman (WL) subtree features.
	We briefly recap the WL relabeling method, formulate the Weisfeiler-Lehman subtree filtration kernel and show that the kernel can in fact be computed in linear time.
	We furthermore show that considering WL labels over certain filtrations increases the expressiveness over the ordinary WL procedure in terms of distinguishing non-isomorphic graphs.

	\subsection{The Weisfeiler-Lehman Method}
		The key idea of the \emph{Weisfeiler-Lehman} procedure is to iteratively aggregate neighborhoods by compressing each node's labels and that of its neighbors into a new label.
		This compression is done by first concatenating a node's label and its ordered (multi-)set of neighbor labels and subsequently hashing this string to a new label using a perfect hash function $f_\#$.
		Thus, with each iteration, a label incorporates increasingly large neighborhoods.  
		More precisely, let $G=(V,E,\ell_0)$ be a graph with initial vertex label function $\ell_0:V \rightarrow \Sigma_0$, where $\Sigma_0$ is the alphabet of original vertex labels. 
		Assuming that there is a total order on alphabet $\Sigma_i$ for all $i \geq 0$, the Weisfeiler-Lehman procedure recursively computes the new label of a node $v$ in iteration $i+1$ by 
		\begin{equation}
			\ell_{i+1}(v) = f_\#(\ell_i(v), [ \ell_i(u): u \in \N(v) ]) \in \Sigma_{i+1}
		\end{equation} 
		where the list of the neighbors' labels is sorted according to the total order on $\Sigma_i$. 
		
		\citet{DBLP:journals/jmlr/ShervashidzeSLMB11} employed the Weisfeiler-Lehman method to define a family of graph kernels of which the \emph{subtree kernel} is perhaps the most popular member. 
		For two graphs $G,G'$, the Weisfeiler-Lehman subtree kernel $K_{WL}$ essentially counts all pairs of mutual node labels. 
		This can be expressed as a histogram kernel on the combined label sets $\F_{WL} = \bigcup_{i \in [h]}\Sigma_i$, i.e., 
		\begin{equation}
			K_{WL}(G,G') = \sum_{l \in \F_{WL}} c(l(G)) ~ c(l(G')) ~~,
		\end{equation} 
		where $c(l(G))$ is the number of appearences of label $l$ in $G$ and $h$ is the depth parameter. 
		
		
		\paragraph{The Weisfeiler-Lehman Isomorphism Test}
		The Weisfeiler-Lehman method was originally designed to decide isomorphism between graphs with one-sided error.
		Two graphs $G,G'$ are not isomorphic if the corresponding multisets $\mso \ell_{i}(v): v \in V(G) \msc$ and $\mso \ell_{i}(v'): v' \in V(G') \msc$ differ for some $i \in \mathbb{N}$; otherwise they may or may not be isomorphic. 
		However, $G$ and $G'$ are isomorphic with high probability if the multisets are equal \citep{babai1979graph}.

	\subsection{The Weisfeiler-Lehman Subtree Filtration Kernel}
	    \label{sec:fwl_kernel}
		Recall that filtrations describe an order in which edges are successively added until the final graph is obtained. 
		During this sequence, neighborhoods of vertices evolve and with them their Weisfeiler-Lehman labels change. 
		Such labels in filtration graphs allow us to consider \emph{partial} neighborhoods and thus also contribute to a \emph{finer} similarity measure between graphs. 
		Fig. \ref{fig:pipeline} depicts an example of a filtration where appearances of a specific WL feature are being tracked. 
		Plugging the WL features $\F_{WL}$ into Eq. \ref{eq:filtration_kernel} yields the \emph{Weisfeiler-Lehman filtration kernel} $K^{\F_{WL}}_{\text{Filt}}$.
		\begin{theorem}
		    \label{thm:complexity}
		    The Weisfeiler-Lehman filtration kernel $K^{\F_{WL}}_{\text{Filt}}(G,G')$ on graphs $G,G'$ can be computed in time $O(hkm)$, where $m$ denotes the number of edges.
		\end{theorem}
		Thus, the kernel increases the complexity of the ordinary Weisfeiler-Lehman subtree kernel merely by the factor $k$, i.e., the length of the filtration. 
		A proof is provided in App. C. 
		

	\subsection{On the Expressive Power of Weisfeiler-Lehman Filtration Kernels}
    	\label{sec:kernel:power}
    	
    	We now show that tracking WL features over filtrations yields powerful methods in terms of expressiveness.
    	The \emph{expressive power} of a method describes its ability to distinguish non-isomorphic graphs; i.e., method $A$ is said to be ``more expressive'' or ``more powerful'' than method $B$ if $A(G) = A(G') \Rightarrow B(G) = B(G')$ and there exist non-isomorphic graphs $G,G'$ such that $A(G) \neq A(G')$ and $B(G) = B(G')$.
    	Recall that the (1-dimensional) WL test for isomorphism is known to be inexact even after $n$ iterations \citep{DBLP:journals/combinatorica/CaiFI92}. 
        However, considering \emph{WL labels over certain filtrations} strictly increases the expressive power. 
        %
        We provide proofs in App. D. 
    	
    	
    	\begin{theorem}\label{thm:power}
    		There exists a filtration function $\A$ such that   
    		 $\phi^\A_f(G)=\phi^\A_f(G')$ for all $f \in \F_{WL}$ if and only if $G$ and $G'$ are isomorphic.
	    \end{theorem}

        In fact, it holds that already for depth-$1$ WL labels (i.e. $h=1$), there are filtrations that \emph{correctly} decide isomorphism for \emph{all} graphs.
    	As a first implication, the Weisfeiler-Lehman subtree filtration kernel is -- given a suitable filtration --  strictly more expressive than the ordinary WL subtree kernel on the original graphs.
    	
    	
    	\begin{corollary}\label{corr:power:kernels}
	    	There exists a filtration function $\A$ such that the kernel $K^{\F_{WL},\A}_{\text{Filt}}$ is \emph{complete}.
    	\end{corollary} 

        A kernel is called \emph{complete} if its feature map $\varphi$ satisfies $\varphi(G)=\varphi(G') \Leftrightarrow G,G'$ are isomorphic \citep{DBLP:conf/colt/GartnerFW03}.
    	While such a filtration is not known to be efficiently computable, there are efficiently computable filtrations that result in strictly more expressive (but incomplete) WL filtration kernels when compared to the ordinary WL subtree kernel.
        %
    	As an example of such an efficiently computable filtration, consider the function that annotates edges by the number of triangles they belong to. 
    	This measure can be computed in polynomial time.
    	Figure~\ref{fig:power:c} shows two 3-regular graphs that can be distinguished using this filtration.\footnote{
    	    It is obvious that using the proposed weights as edge labels allows the Weisfeiler-Lehman isomorphism test to distinguish these two edge labeled graphs.
    	    However, in App. D
    	    , we show that it suffices to consider the WL labels on the unlabeled filtration graphs. 
    	}
    	This concept can be extended to larger (or multiple) subgraphs, which allows for more expressive filtrations, similar to the approach of \citet{Bouritsas2020gnnwithisomorphismcounts}.
    	
    	\begin{figure}
    	    \centering
    	    \begin{tikzpicture}[baseline=(current bounding box.south)]
    \newcommand{\edgelength}[0]{0.58}
    \newcommand{\skeleton}[1]{
        \coordinate (x) at #1;
        \node[circle, draw=black, inner sep=2.3pt] (a) at ($ (x)  $) {}; 
        \node[circle, draw=black, inner sep=2.3pt] (b) at ($ (x) + (210:\edgelength) $) {}; 
        \node[circle, draw=black, inner sep=2.3pt] (c) at ($ (x) + (270:\edgelength) $) {}; 
        \node[circle, draw=black, inner sep=2.3pt] (d) at ($ (c) + (0:\edgelength) $) {}; 
        \node[circle, draw=black, inner sep=2.3pt] (e) at ($ (d) + (30:\edgelength) $) {}; 
        \node[circle, draw=black, inner sep=2.3pt] (f) at ($ (d) + (90:\edgelength) $) {}; 
    }
    \newcommand{\descriptionnode}[1]{\node (description) at ($ (f) + (-0.5 * \edgelength, 1.0 * \edgelength) $) {\small #1};}
    \newcommand{\roundbox}{\draw[rounded corners, draw=gray] ($ (x) + (-1.2 * \edgelength, 0.7 * \edgelength) $) rectangle ($ (d) + (1.2 * \edgelength, -0.5 *\edgelength) $);}
    \newcommand{\rounddoublebox}{\draw[rounded corners, draw=gray] ($ (x) + (-1.2 * \edgelength, 0.7 * \edgelength) $) rectangle ($ (d) + (4.6 * \edgelength, -0.5 *\edgelength) $);
    \draw[dash pattern=on 1pt off 1pt, draw=gray] ($ (d) + (1.2 * \edgelength, -0.5 *\edgelength) $) -- ($ (f) + (1.2 * \edgelength, 0.7 *\edgelength) $);}

    \skeleton{(0, 0)}
    \draw[-] (a) -- node[above=-0.05] {\tiny 1} (b) 
                 -- node[below=-0.05] {\tiny 1} (c) 
                 -- node[right=-0.05] {\tiny 1} (a)
             (d) -- node[below=-0.05] {\tiny 1} (e) 
                 -- node[above=-0.05] {\tiny 1} (f) 
                 -- node[left=-0.05] {\tiny 1} (d)
             (c) -- node[below=-0.05] {\tiny 0} (d)  
             (f) -- node[above=-0.05] {\tiny 0} (a);
    \draw[-] (b) .. controls ($ (a) + (135:\edgelength) $) and ($ (f) + (45:\edgelength) $) .. node[left=0.45] {\tiny 0} (e);

    \descriptionnode{$G$};
    \roundbox
    \node[anchor=west] at ($ (e) + (0.45 * \edgelength, 0 * \edgelength) $) {\tiny $\to$};
    
    \coordinate (bbnw) at ($ (description) + (-0.45 * \edgelength, 0.45 * \edgelength) $); 

    \skeleton{($ (0, -1.8) $)}
    \draw[-] (a) -- node[above=-0.05] {\tiny 0} (b) 
                 -- node[below=-0.05] {\tiny 0} (c) 
                 -- node[below=-0.05] {\tiny 0} (d)
                 -- node[below=-0.05] {\tiny 0} (e) 
                 -- node[above=-0.05] {\tiny 0} (f) 
                 -- node[above=-0.05] {\tiny 0} (a)
             (c) -- node[left=-0.05] {\tiny 0} (f)  
             (d) -- node[right=-0.05] {\tiny 0} (a);
    \draw[-] (b) .. controls ($ (a) + (135:\edgelength) $) and ($ (f) + (45:\edgelength) $) .. node[left=0.45] {\tiny 0} (e); 
    
    \descriptionnode{$G'$};
    \roundbox 
    \node[anchor=west] at ($ (e) + (0.45 * \edgelength, 0 * \edgelength) $) {\tiny $\to$};

    
    \skeleton{($ (4.5 * \edgelength, 0) $)}
    \draw[-] (a) -- (b) -- (c) -- (a) 
             (d) -- (e) -- (f) -- (d);
             
    \descriptionnode{$G_1$};
    \rounddoublebox

    \skeleton{($ (4.5 * \edgelength, -1.8) $)}
    
    \descriptionnode{$G'_1$};
    \rounddoublebox

    
    \skeleton{($ (7.9 * \edgelength, 0) $)}
    \draw[-] (a) -- (b) 
                 -- (c) 
                 -- (a)
             (d) -- (e) 
                 -- (f) 
                 -- (d)
             (c) -- (d)  
             (f) -- (a);
    \draw[-] (b) .. controls ($ (a) + (135:\edgelength) $) and ($ (f) + (45:\edgelength) $) .. (e);

    \descriptionnode{$G_2$};

    \skeleton{($ (7.9 * \edgelength, -1.8) $)}
    \draw[-] (a) -- (b) 
                 -- (c) 
                 -- (d)
                 -- (e) 
                 -- (f) 
                 -- (a)
             (c) -- (f)  
             (d) -- (a);
    \draw[-] (b) .. controls ($ (a) + (135:\edgelength) $) and ($ (f) + (45:\edgelength) $) .. (e); 
    
    \descriptionnode{$G'_2$};
    \coordinate (bbse) at ($ (b) + (1 * \edgelength, -0.33 * \edgelength) $); 
    
    \useasboundingbox (bbnw) rectangle (bbse);
    
\end{tikzpicture}
    	    \caption{Consider the unlabeled 3-regular graphs $G,G'$ which cannot be distinguished using the Weisfeiler-Lehman isomorphism test.
    	    However, adding edge weights that count the number of triangles that the edges are part of yields filtration graphs $G_1,G_1'$. These can now clearly be distinguished by the WL isomorphism test. }
    	    \label{fig:power:c}
    	\end{figure}
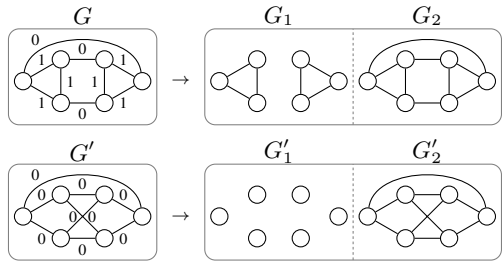
    	
    	Recently, there has been a large body of work that relates the expressive power of certain graph neural networks to that of the Weisfeiler-Lehman isomorphism test \citep{DBLP:conf/iclr/XuHLJ19, morris2019wlneural}.
    	While this is orthogonal to our work, Theorem~\ref{thm:power} implies that a (neural) ensemble of graph neural networks over a suitably chosen filtration is strictly more powerful than a graph neural network on the original graph(s) alone:
    	
    	\begin{corollary}\label{corr:power:gnns}
    	    There exists a filtration function $\A$ and GNN \citep{DBLP:conf/iclr/XuHLJ19} $\mathcal{N}$ such that $\mathcal{N}$ can distinguish any two non-isomorphic graphs when provided with the filtration graphs corresponding to $\A$. 
    	\end{corollary}
	
	\section{Related Work}
\label{sec:related}

\paragraph{Graph Kernels}
Over the past twenty years there has been a multitude of works constructing kernels for graph structured data, with many well known examples within the $\mathcal{R}$-convolution framework \citep{Haussler99,DBLP:conf/colt/GartnerFW03,DBLP:conf/icdm/BorgwardtK05,DBLP:journals/jmlr/ShervashidzeVPMB09,DBLP:journals/jmlr/ShervashidzeSLMB11,DBLP:conf/kdd/HorvathGW04}.
We refer to \citet{DBLP:journals/ftml/BorgwardtGLOR20,DBLP:journals/ans/KriegeJM20} for recent surveys.
The seminal work by \citet{DBLP:journals/jmlr/ShervashidzeSLMB11} introduce Weisfeiler-Lehman labels as a means to define graph kernels.
\citet{DBLP:conf/nips/KriegeGW16} propose a discrete optimal assignment graph kernel based on vertex kernels obtained from the hierarchy of their WL labels.
\citet{DBLP:conf/nips/TogninalliGLRB19} extend this idea and allow fractional assignments using Wasserstein distances.
\citet{DBLP:conf/icml/RieckBB19} introduce a method which individually weighs WL label occurences by their persistent homology information. 
The difference to our approach is that we consider arbitrary graph features. 
Furthermore, in the particular case of Weisfeiler-Lehman features, we allow additional WL features appearing only in filtration graphs.
Additionally, they provide a kernel based on persistence tuple distances, which is, however, restricted to non-attributed graphs. 
\citet{DBLP:conf/ijcai/NikolentzosMLV18} compute $k$-core decompositions to compare graphs using such hierarchy of subgraphs.
While conceptually similar to our approach, the method is limited to nested subgraph chains generated by $k$-cores. 
More importantly, the authors revert to applying existing graph kernels to the graphs' $k$-cores for each $k$ and thus merely compare subgraphs of the same granularities, whereas we introduce a family of novel graph kernels that compare feature distributions over multiple scales. 



\paragraph{Expressiveness}
In pioneer work, \citet{DBLP:conf/colt/GartnerFW03} have shown that \emph{complete} graph kernels, i.e., kernels that map isomorphism classes to distinct points in some Hilbert space, are at least as hard to compute as graph isomorphism.
While not complete, the Weisfeiler-Lehman subtree kernel maps  graphs that can be distinguished by the WL isomorphism test to different feature vectors and is hence equally expressive. 
Recently, \citet{DBLP:conf/iclr/XuHLJ19} and \citet{morris2019wlneural,morris2020wlsparse, morris2021wlpower} have investigated the connection between the expressive power of graph neural networks and the WL isomorphism test. 
Briefly speaking, if the ($k$-dimensional) Weisfeiler-Lehman isomorphism test cannot distinguish two graphs then (higher order) graph neural networks cannot distinguish these graphs, either, and vice versa.
Recently, however, there has been significant work in extending the expressive power of graph neural networks by adding auxiliary information such as subgraph isomorphism counts \citep{Bouritsas2020gnnwithisomorphismcounts} or distance encodings \citep{Li2020gnnwithdistanceencoding}.

\enlargethispage*{1em}
	

\newcommand{\WL}[0]{\text{WL\xspace}}
\newcommand{\WWL}[0]{\text{W-WL}\xspace}
\newcommand{\PWL}[0]{\text{P-WL}\xspace}
\newcommand{\FPWLdeg}[0]{\text{FWL-D}\xspace}
\newcommand{\FPWLrw}[0]{\text{FWL-RW}\xspace}
\newcommand{\GS}[0]{\text{GS}\xspace}
\newcommand{\SP}[0]{\text{SP}\xspace}
\newcommand{\WLOA}[0]{\text{WL-OA}\xspace}
\newcommand{\BL}[0]{\text{VE-Hist}\xspace}
\newcommand{\CoreWL}[0]{\text{Core-WL}\xspace}

\definecolor{baselinecolor}{RGB}{1,133,113}
\definecolor{GScolor}{RGB}{102,205,170}
\definecolor{SPcolor}{RGB}{95,158,160}
\definecolor{WLcolor}{RGB}{205,92,92}
\definecolor{WLOAcolor}{RGB}{255,150,100} 
\definecolor{WWLcolor}{RGB}{166,97,26}
\definecolor{PWLcolor}{RGB}{223,194,125}
\definecolor{CoreWLcolor}{RGB}{255,128,0}
\definecolor{FPWLdegcolor}{RGB}{72,61,139}
\definecolor{FPWLrwcolor}{RGB}{70,130,180}

\newcommand{\WLmark}[0]{square*}
\newcommand{\WWLmark}[0]{halfsquare*}
\newcommand{\PWLmark}[0]{diamond*}
\newcommand{\FPWLdegmark}[0]{halfcircle*}
\newcommand{\FPWLrwmark}[0]{*}
\newcommand{\GSmark}[0]{x}
\newcommand{\SPmark}[0]{star}
\newcommand{\WLOAmark}[0]{pentagon*}
\newcommand{\BLmark}[0]{triangle*}

\newcommand{\barlights}[1]{#1!70!white}

\section{Experimental Evaluation}
	In this section, we evaluate the predictive performance of our Weisfeiler-Lehman filtration kernel\footnote{Available at \url{ https://github.com/mlai-bonn/wl-filtration-kernel}} introduced above and compare it to several state-of-the-art graph kernels. 
	We demonstrate that our method significantly outperforms its competitors on several real-world benchmark datasets.

	\paragraph{Experimental Setup \& Datasets}
		We compare our method to a variety of state-of-the-art graph kernels and include a simple baseline method to put the results into perspective. 
		The experiments are conducted on the well-established molecular datasets DHFR, NCI1 and PTC-MR \citep[obtained from][]{DBLP:journals/corr/abs-2007-08663} as well as the large network benchmark datasets IMDB-BINARY \citep[obtained from][]{DBLP:journals/corr/abs-2007-08663} and EGO-1 to EGO-4. 
		This selection contains only such datasets on which a simple histogram baseline kernel was outperformed by more sophisticated graph kernels. 
		We measure the accuracies obtained by support vector machines (SVM) using a $10$-fold stratified cross-validation. 
		A grid search over sets of kernel specific parameters is used for optimal training. 
		We perform $10$ such cross-validations and report the mean and standard deviation. 
		More detailed information is available in App. E. 

		
	\subsection{Filtration Variants}
	\label{sec:experiments:filtrationdefinition}
		A unique parameter of our kernels is the graph filtration. 
		If such a filtration is not provided through expert knowledge or via available edge weights, one can be generated based on the data at hand. 
		Recall that a filtration is induced by the value sequence $\alpha_1 \geq \ldots \geq \alpha_k$. 
		While there exist infinitely many such sequences, in the following experiments, we govern only its length (i.e. $k$) and generate the $\alpha_i$s according to some fixed process. 
		In this work, we implement this process by first ordering and subsequently partitioning the (multi)-set of edge weights that appear in a given dataset using a simple $k$-means clustering. 
		Value $\alpha_i$ is then chosen as the minimum element of the $i$-th cluster.
		Thus, the $i$-th filtration graph $G_i$ contains edges with edge weight at least $\alpha_i$.
		
		As the above datasets are not explicitly equipped with edge weights, we consider two exemplary edge weight functions $w_{d}$ and $w^\lambda_{rw}$ that each assign weights to edges according to the edges' structural relevance w.r.t. a specific property. 
		    The first function $w_{d}$ assigns an edge $e=\{u,v\}$ a weight equal to the maximum degree of its incident vertices, i.e., $w_d(e) = max(deg(u), deg(v))$. 
		    Thus, edges with incident vertices of high degree appear early in the graph filtrations. 
		    The second function $w^\lambda_{rw}$ considers the number of random walks between adjacent vertices. 
		    That is, for edge $e=\{u,v\}$, $w^\lambda_{rw}(e)$ is the number of random walks of length at most $\lambda$ from $u$ to $v$.
		
		In the following, the Weisfeiler-Lehman filtration kernel applied on dataset graphs with edge weights calculated according to $w_{d}$ is referred to as \FPWLdeg.
		Analogously, \FPWLrw denotes the kernel variant based on weights computed according to $w^\lambda_{rw}$.
		\enlargethispage*{1em}

	\subsection{Real-World Benchmarks}
		Figure \ref{fig:benchmarks} compares the predictive performances of the filtration kernel to various state-of-the-art graph kernels on a range of real-world benchmark datasets. 
		On datasets NHFR, NCI1, PTC-MR and IMDB-BINARY, there are only small discrepancies between our method and the best performing competitor kernels.
		In fact, with PTC-MR being an exception, the results of all kernels using Weisfeiler-Lehmann subtree features are virtually indistinguishable. 
		This changes for the EGO datasets. 
		While on EGO-1, our \FPWLdeg variant is second only to the shortest-path kernel, it significantly outperforms all tested kernels on EGO-2, EGO-3 and EGO-4, amounting to a roughly $10\%$ accuracy increase for the case of the latter dataset. 
		It becomes apparent from the results of the \FPWLrw variant that our approach is particularly reliant on the provided edge weights. 
		However, the evaluation also suggests that in case the data at hand is not equipped by meaningful edge weights, very simple edge weight functions already suffice to significantly increase the predictive performance over state-of-the-art graph kernels. 
		
		\begin{figure*}[t]
			\begin{center}
				\resizebox{1.0\textwidth}{!}{\begin{tikzpicture}
\begin{axis}[
	width=24cm,
	height=9cm,
	bar shift auto,
	bar width=0.15cm,
	legend style={at={(0.5,1.0)}, anchor=north,legend columns=-1, /tikz/every even column/.append style={column sep=0.25cm}},
    xtick={1,...,9},
    xticklabels={DHFR, NCI1, PTC-MR, IMDB-B., EGO-1, EGO-2, EGO-3, EGO-4},
    ymajorgrids,
    ybar,
    ylabel={Accuracy in \%},
    ylabel near ticks,
    ymin=20,
    ymax=95
    ]

\addplot[
	baselinecolor, 
	fill=baselinecolor, 
	draw=baselinecolor, 
	error bars/.cd, 
	y dir=both, 
	y explicit, 
	error bar style={draw=\barlights{baselinecolor}}, 
	error mark options={mark size=2pt, rotate=90, draw=\barlights{baselinecolor}}, 
]
table [y error=error] {
	x   y		error    
    1	72.86	0.98
    2	69.82	0.40
    3	56.88	1.64
    4	70.36	0.84
    5	29.00	3.22
    6	26.30	2.73
    7	23.75	2.24
    8	26.40	2.37
};

\addplot[
	GScolor, 
	fill=GScolor, 
	draw=GScolor, 
	error bars/.cd, 
	y dir=both, 
	y explicit, 
	error bar style={draw=\barlights{GScolor}}, 
	error mark options={mark size=2pt, rotate=90, draw=\barlights{GScolor}}, 
]
table [y error=error] {
	x   y       error    
	1	60.77	0.31
    2	60.54	1.34
    3	55.82	1.62
    4	66.62	0.77
    5	60.45	2.31
    6	52.90	2.25
    7	51.85	2.93
    8	51.40	3.75

};

\addplot[
	SPcolor, 
	fill=SPcolor, 
	draw=SPcolor, 
	error bars/.cd, 
	y dir=both, 
	y explicit, 
	error bar style={draw=\barlights{SPcolor}}, 
	error mark options={mark size=2pt, rotate=90, draw=\barlights{SPcolor}}, 
]
table [y error=error] {
	x   y           error    
	1	78.36	0.81
    2	73.69	0.24
    3	57.23	1.93
    4	49.50	1.38
    5	66.80	2.14
    6	57.40	1.70
    7	59.80	1.18
    8	59.30	2.04
};

\addplot[
	WLcolor, 
	fill=WLcolor, 
	draw=WLcolor, 
	error bars/.cd, 
	y dir=both, 
	y explicit, 
	error bar style={draw=\barlights{WLcolor}}, 
	error mark options={mark size=2pt, rotate=90, draw=\barlights{WLcolor}}, 
]
table [y error=error] {
	x   y           error    
	1	82.34	0.86
    2	85.67	0.23
    3	61.07	1.78
    4	72.12	0.79
    5	51.90	1.81
    6	58.00	2.22
    7	56.50	2.26
    8	54.05	2.23 
};

\addplot[
	WLOAcolor, 
	fill=WLOAcolor, 
	draw=WLOAcolor, 
	error bars/.cd, 
	y dir=both, 
	y explicit, 
	error bar style={draw=\barlights{WLOAcolor}}, 
	error mark options={mark size=2pt, rotate=90, draw=\barlights{WLOAcolor}}, 
]
table [y error=error] {
	x   y           error    
	1	82.92	0.69
    2	85.91	0.27
    3	61.85	1.72
    4	73.06	0.65
    5	58.25	1.23
    6	63.50	1.80
    7	57.80	1.80
    8	55.35	1.90
};

\addplot[
	WWLcolor, 
	fill=WWLcolor, 
	draw=WWLcolor, 
	error bars/.cd, 
	y dir=both, 
	y explicit, 
	error bar style={draw=\barlights{WWLcolor}}, 
	error mark options={mark size=2pt, rotate=90, draw=\barlights{WWLcolor}}, 
]
table [y error=error] {
	x   y           error    
    1	82.10	00.70
    2	85.70	00.46
    3	66.30	01.62
    4	73.30	00.78
    5	59.90	01.14
    6	63.10	02.07
    7	58.40	01.36
    8	55.80	01.99
};

\addplot[
	PWLcolor, 
	fill=PWLcolor, 
	draw=PWLcolor, 
	error bars/.cd, 
	y dir=both, 
	y explicit, 
	error bar style={draw=\barlights{PWLcolor}}, 
	error mark options={mark size=2pt, rotate=90, draw=\barlights{PWLcolor}}, 
]
table [y error=error] {
	x   y           error    
    1	83.9610000000	0.2319245567
    2	85.6810000000	0.0805543295
    3	64.7020000000	0.4205425068
    4	72.0000000000	0.2049878045
    5	60.8950000000	0.6520927848
    6	58.9300000000	0.5573149917
    7	52.4100000000	0.7729165544
    8	53.2050000000	0.8594911285
};

\addplot[
	CoreWLcolor, 
	fill=CoreWLcolor, 
	draw=CoreWLcolor, 
	error bars/.cd, 
	y dir=both, 
	y explicit, 
	error bar style={draw=\barlights{CoreWLcolor}}, 
	error mark options={mark size=2pt, rotate=90, draw=\barlights{CoreWLcolor}}, 
]
table [y error=error] {
	x   y           error    
    1	83.45	0.55
    2	85.48	0.37
    3	58.76	2.36
    4	73.54	0.78
    5	57.10	2.17
    6	57.40	2.22
    7	51.10	2.62
    8	53.05	2.24
};

\addplot[
	FPWLdegcolor, 
	fill=FPWLdegcolor, 
	draw=FPWLdegcolor, 
	error bars/.cd, 
	y dir=both, 
	y explicit, 
	error bar style={draw=\barlights{FPWLdegcolor}}, 
	error mark options={mark size=2pt, rotate=90, draw=\barlights{FPWLdegcolor}}, 
]
table [y error=error] {
	x   y       error    
	1	82.75	0.77
    2	85.45	0.27
    3	62.44	1.32
    4	72.22	0.87
    5	63.60	2.83
    6	67.60	1.26
    7	66.65	1.13
    8	69.40	1.51
};

\addplot[
	FPWLrwcolor, 
	fill=FPWLrwcolor, 
	draw=FPWLrwcolor, 
	error bars/.cd, 
	y dir=both, 
	y explicit, 
	error bar style={draw=\barlights{FPWLrwcolor}}, 
	error mark options={mark size=2pt, rotate=90, draw=\barlights{FPWLrwcolor}}, 
]
table [y error=error] {
	x   y           error    
    1	81.93	0.65
    2	85.67	0.21
    3	61.22	1.72
    4	72.73	0.74
    5	50.95	1.34
    6	58.70	1.58
    7	57.55	2.54
    8	54.10	1.26
};

\legend{\BL, \GS, \SP, \WL, \WLOA, \WWL, \PWL, \CoreWL, \FPWLdeg, \FPWLrw}
\path (axis cs:1,0) -- coordinate (m) (axis cs:2,0);
\draw[lightgray] (m) -- (current axis.north -| m);
\path (axis cs:2,0) -- coordinate (m) (axis cs:3,0);
\draw[lightgray] (m) -- (current axis.north -| m);
\path (axis cs:3,0) -- coordinate (m) (axis cs:4,0);
\draw[lightgray] (m) -- (current axis.north -| m);
\path (axis cs:4,0) -- coordinate (m) (axis cs:5,0);
\draw[lightgray] (m) -- (current axis.north -| m);
\path (axis cs:5,0) -- coordinate (m) (axis cs:6,0);
\draw[lightgray] (m) -- (current axis.north -| m);
\path (axis cs:6,0) -- coordinate (m) (axis cs:7,0);
\draw[lightgray] (m) -- (current axis.north -| m);
\path (axis cs:7,0) -- coordinate (m) (axis cs:8,0);
\draw[lightgray] (m) -- (current axis.north -| m);
\end{axis}
\end{tikzpicture}}
			\end{center}
			\caption{Classification accuracies and standard deviations on real-world benchmark datasets. 
			We compare our approach to a simple baseline kernel (\BL), the graphlet sampling (\GS) kernel \citep{DBLP:journals/jmlr/ShervashidzeVPMB09}, the shortest-path (\SP) kernel \citep{DBLP:conf/icdm/BorgwardtK05}, the original Weisfeiler-Lehman (\WL) kernel \citep{DBLP:journals/jmlr/ShervashidzeSLMB11}, the Weisfeiler-Lehman optimal assignment (\WLOA) kernel \citep{DBLP:conf/nips/KriegeGW16}, the Wasserstein Weisfeiler-Lehman (\WWL) kernel \citep{DBLP:conf/nips/TogninalliGLRB19}, the persistent Weisfeiler-Lehman (\PWL) method \citep{DBLP:conf/icml/RieckBB19}, and the core variant of the Weisfeiler-Lehman (\CoreWL) kernel \citep{DBLP:conf/ijcai/NikolentzosMLV18}. 
			We employ our WL filtration kernels (FWL) using two different edge weight functions on graphs (see Sect. \textit{Filtration Variants}).} 
			\label{fig:benchmarks}
		\end{figure*}
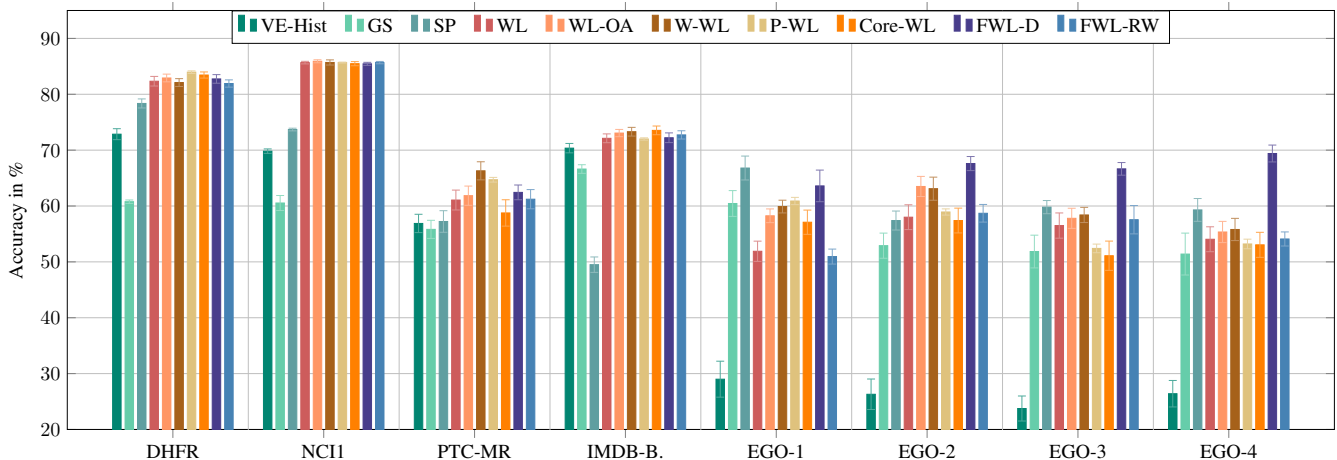

	\subsection{The Influence of the Filtration Length $k$}
		As central aspect of our approach, the filtration clearly plays a critical role in practical applications.
		In order to investigate the influence of filtrations, we provide experiments for varying choices of $k$, i.e., the filtration length.
		Figure \ref{fig:k_invest} shows results on the EGO datasets for values $k \in \{1,\ldots,10\}$ using the \FPWLdeg{} variant.
		Recall that the case $k=1$ is equivalent to the ordinary Weisfeiler-Lehman approach. 
		Note that already for $k=2$ the predictive performance significantly improves over $k=1$ in all cases and that all datasets reach their accuracy peak at only $k=2$ or $k=3$. 
		App. C 
		shows that the value $k$ linearly influences the kernel complexity.
		Thus, our approach can improve the predictive performance over the ordinary Weisfeiler-Lehman subtree kernel at cost of only a minimal increase in complexity. 
		
		\begin{figure}
            \centering
            \resizebox{0.45\textwidth}{!}{\begin{tikzpicture}
\begin{axis}[
	width=12cm,
	height=7cm,
	xtick={1,...,10},
	ymin=50, ymax=73,
	grid=major,
	grid style={line width=.5pt, draw=gray!20},
	legend pos=north west,
	ylabel={Accuracy in \%},
	ylabel near ticks,
	xlabel={$k$},
	legend style={at={(1,0)},anchor=south east}
]

\addplot[color=baselinecolor, line width=0.25mm] 
table [y error=error] {
	x   y		error    
	1	51.50	1.84
	2	59.85	2.40
	3	63.85	2.84
	4	60.70	2.36
	5	59.60	2.56
	6	59.20	1.46
	7	59.85	1.76
	8	60.25	2.23
	9	59.55	2.41
	10	59.55	2.36
};

\addplot[name path=ego1_top,color=baselinecolor!30, draw=none, forget plot] 
table [y error=error] {
	x   y		error    
	1	49.66	0.00
	2	57.45	0.00
	3	61.01	0.00
	4	58.34	0.00
	5	57.04	0.00
	6	57.74	0.00
	7	58.09	0.00
	8	58.02	0.00
	9	57.14	0.00
	10	57.19	0.00
};

\addplot[name path=ego1_bottom,color=baselinecolor!30, draw=none, forget plot] 
table [y error=error] {
	x   y		error    
	1	53.34	0.00
	2	62.25	0.00
	3	66.69	0.00
	4	63.06	0.00
	5	62.16	0.00
	6	60.66	0.00
	7	61.61	0.00
	8	62.48	0.00
	9	61.96	0.00
	10	61.91	0.00
};

\addplot[baselinecolor!50,fill opacity=0.5, forget plot] fill between[of=ego1_top and ego1_bottom];

\addplot[color=CoreWLcolor] 
table [y error=error] {
	x   y		error    
	1	57.75	1.46
	2	68.50	1.84
	3	63.05	2.64
	4	64.55	1.62
	5	68.50	2.07
	6	65.90	2.40
	7	69.00	2.32
	8	69.05	2.51
	9	66.75	2.85
	10	67.55	2.15
};

\addplot[name path=ego2_top,color=CoreWLcolor!70, draw=none, forget plot] 
table [y error=error] {
	x   y		error    
	1	56.29	0.00
	2	66.66	0.00
	3	60.41	0.00
	4	62.93	0.00
	5	66.43	0.00
	6	63.50	0.00
	7	66.68	0.00
	8	66.54	0.00
	9	63.90	0.00
	10	65.40	0.00
};

\addplot[name path=ego2_bottom,color=CoreWLcolor!70, draw=none, forget plot] 
table [y error=error] {
	x   y		error    
	1	59.21	0.00
	2	70.34	0.00
	3	65.69	0.00
	4	66.17	0.00
	5	70.57	0.00
	6	68.30	0.00
	7	71.32	0.00
	8	71.56	0.00
	9	69.60	0.00
	10	69.70	0.00
};

\addplot[CoreWLcolor!50,fill opacity=0.5, forget plot] fill between[of=ego2_top and ego2_bottom];

\addplot[color=FPWLdegcolor, line width=0.25mm] 
table [y error=error] {
	x   y		error    
	1	58.20	1.77
	2	67.20	1.55
	3	65.75	1.87
	4	65.50	2.57
	5	65.45	2.06
	6	64.60	2.56
	7	62.45	2.49
	8	62.50	1.89
	9	61.95	2.35
	10	63.45	1.98
};

\addplot[name path=ego3_top,color=FPWLdegcolor!30, draw=none, forget plot] 
table [y error=error] {
	x   y		error    
	1	56.43	0.00
	2	65.65	0.00
	3	63.88	0.00
	4	62.93	0.00
	5	63.39	0.00
	6	62.04	0.00
	7	59.96	0.00
	8	60.61	0.00
	9	59.60	0.00
	10	61.47	0.00
};

\addplot[name path=ego3_bottom,color=FPWLdegcolor!30, draw=none, forget plot] 
table [y error=error] {
	x   y		error    
	1	59.97	0.00
	2	68.75	0.00
	3	67.62	0.00
	4	68.07	0.00
	5	67.51	0.00
	6	67.16	0.00
	7	64.94	0.00
	8	64.39	0.00
	9	64.30	0.00
	10	65.43	0.00
};

\addplot[FPWLdegcolor!50,fill opacity=0.5, forget plot] fill between[of=ego3_top and ego3_bottom];

\addplot[color=WLcolor, line width=0.25mm] 
table [y error=error] {
	x   y		error    
	1	55.10	1.70
	2	62.35	1.40
	3	69.60	1.54
	4	66.85	1.94
	5	63.15	2.16
	6	65.80	1.42
	7	63.60	1.24
	8	66.20	2.58
	9	66.25	2.37
	10	65.25	2.32
};

\addplot[name path=ego4_top,color=WLcolor!30, draw=none, forget plot] 
table [y error=error] {
	x   y		error    
	1	53.40	0.00
	2	60.95	0.00
	3	68.06	0.00
	4	64.91	0.00
	5	60.99	0.00
	6	64.38	0.00
	7	62.36	0.00
	8	63.62	0.00
	9	63.88	0.00
	10	62.93	0.00
};

\addplot[name path=ego4_bottom,color=WLcolor!30, draw=none, forget plot] 
table [y error=error] {
	x   y		error    
	1	56.80	0.00
	2	63.75	0.00
	3	71.14	0.00
	4	68.79	0.00
	5	65.31	0.00
	6	67.22	0.00
	7	64.84	0.00
	8	68.78	0.00
	9	68.62	0.00
	10	67.57	0.00
};

\addplot[WLcolor!50,fill opacity=0.5, forget plot] fill between[of=ego4_top and ego4_bottom];

\legend{EGO-1, EGO-2, EGO-3, EGO-4}
\end{axis}
\end{tikzpicture}}
            \captionof{figure}{Classification accuracies and standard deviations of the \FPWLdeg variant for different filtration lengths $k$.}
            \label{fig:k_invest}
        \end{figure}
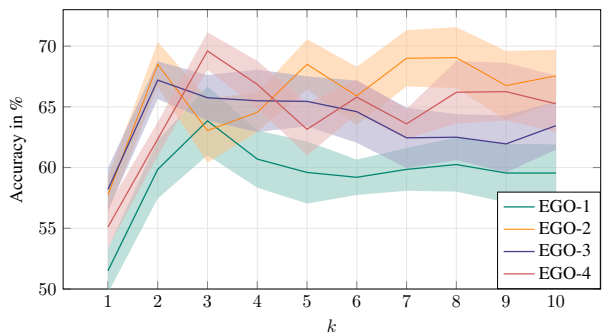%
	

	\subsection{Synthetic Benchmarks: Circular Skip Link Graphs}
		
		In this section, we investigate the discriminate power of filtration graphs kernels. 
		In particular, we consider a benchmark setup as discussed in \citet{DBLP:conf/icml/Murphy0R019} which classifies \emph{circular skip link} (CSL) graphs. 
		A CSL graph $G_{n,s}$ is a $4$-regular graph where the $n$ nodes form a cycle and all pairs of nodes with path distance $s$ on that cycle are furthermore connected by an edge. 
		Examples for CSL graphs are depicted in Fig.~\ref{fig:csl_graphs}.
		Following \citet{DBLP:conf/icml/Murphy0R019}, for $n=41$ and $s \in \{2,3,4,5,6,9,11,12,13,16\}$, graphs are pairwise non-isomorphic. 
		We construct a dataset containing $10$ permuted copies of each such graph.
		The underlying classification task is then to assign graphs to their skip link value $s$ which measures a model's capability to distinguish non-isomorphic graphs. 
		
		
		In this experiment, we consider the \FPWLrw variant and do not limit the number of filtrations $k$ but allow as many filtrations as there are distinct edge weights in the dataset graphs. 
		Thus, the number of filtration graphs is directly governed by $\lambda$, i.e.,  the parameter of $w^\lambda_{rw}~$ which denotes the maximal random walk length. 
		Table~\ref{fig:csl_table} shows the predictive performances for different choices of $\lambda$ (and thus $k$). 
		The case $\lambda=1$ corresponds to the ordinary Weisfeiler-Lehman kernel as $k=1$. 
		Since the Weisfeiler-Lehman method falls short of distinguishing regular graphs, the predictive performance corresponds to that of a random classifier. 
		However, for increasing values of $\lambda$, the number of filtrations $k$ grows as well which results in the kernel's ability to distinguish more and more CSL graphs. 
		Finally, for the case $\lambda=7$, all non-isomorphic graphs can be told apart. 
		It is noteworthy, that these results are not entirely surprising as the considered edge weights provide the filtration kernel with increasing degrees of cyclic information. 
		Nonetheless, the experiments highlight the power of filtration kernels and practically support Corollary \ref{corr:power:kernels}.
        \enlargethispage*{1em}		
		
		\begin{figure}
            \centering
            \resizebox{0.22\textwidth}{!}{\begin{tikzpicture}

	\node (a) at (0,0)
	{
		\begin{tikzpicture}
			\node[text width=0cm] at (-1,1.4) {$G_{11,2}:$};
			\draw[] (0:1) arc (0:360:10mm);
			\foreach \phi in {1,...,11}{
				\node[state,fill=white,scale=0.225] (v_\phi) at (360/11 * \phi:10mm) {};
			}
			\draw[] (v_1) edge[bend left] (v_3);
			\draw[] (v_2) edge[bend left] (v_4);
			\draw[] (v_3) edge[bend left] (v_5);
			\draw[] (v_4) edge[bend left] (v_6);
			\draw[] (v_5) edge[bend left] (v_7);
			\draw[] (v_6) edge[bend left] (v_8);
			\draw[] (v_7) edge[bend left] (v_9);
			\draw[] (v_8) edge[bend left] (v_10);
			\draw[] (v_9) edge[bend left] (v_11);
			\draw[] (v_10) edge[bend left] (v_1);
			\draw[] (v_11) edge[bend left] (v_2);
		\end{tikzpicture}
	};

	\node (b) at (3,0)
	{
		\begin{tikzpicture}
			\node[text width=0cm] at (-1,1.4) {$G_{11,3}:$};
			\draw[] (0:1) arc (0:360:10mm);
			\foreach \phi in {1,...,11}{
				\node[state,fill=white,scale=0.225] (v_\phi) at (360/11 * \phi:10mm) {};
			}
			\draw[] (v_1) edge[bend left] (v_4);
			\draw[] (v_2) edge[bend left] (v_5);
			\draw[] (v_3) edge[bend left] (v_6);
			\draw[] (v_4) edge[bend left] (v_7);
			\draw[] (v_5) edge[bend left] (v_8);
			\draw[] (v_6) edge[bend left] (v_9);
			\draw[] (v_7) edge[bend left] (v_10);
			\draw[] (v_8) edge[bend left] (v_11);
			\draw[] (v_9) edge[bend left] (v_1);
			\draw[] (v_10) edge[bend left] (v_2);
			\draw[] (v_11) edge[bend left] (v_3);
		\end{tikzpicture}
	};

\end{tikzpicture}}
            \captionof{figure}{Circular skip link graphs.}
            \label{fig:csl_graphs}
            \vspace{3mm}
            \resizebox{0.48\textwidth}{!}{
			    \begin{tabular}{c|cccccccccccccc} 
 					$\lambda:k$ & $1:1$ & $2:3$ & $3:4$ & $4:7$ & $5:9$ & $6:14$ & $7:18$ & $8:19$\\ 
 					\hline 
 					Acc. & $10\%~$ & $20\%~$ & $30\%~$ & $50\%~$ & $70\%$ & $90\%$ & $100\%$ & $100\%$ \\ 
 				\end{tabular} }
 			\caption{Classification accuracies of the \FPWLrw variant on CSL graphs.}
 			\label{fig:csl_table}
        \end{figure}

\section{Conclusion}\label{sec:conclusion}

We introduced filtration kernels, a family of graph kernels which compare graphs at different levels of resolution by tracking existence intervals of features over graph filtrations. 
We showed that a particular member of this family enriches the expressive power of the WL subtree kernel.
The implications of this result extend beyond graph kernels, and allow to construct more powerful graph neural networks, as the expressivity of such networks is limited by the WL method.

Empirically, we have demonstrated that our proposed kernel shows comparable or improved predictive performances with respect to other state-of-the-art methods on real-world benchmark datasets.
Given the theoretical insights mentioned above, a particularly interesting research question is concerned with practical performances of filtration-aware graph neural networks.


\subsection*{Acknowledgements}
    This material was produced within the Competence Center for Machine Learning Rhine-Ruhr (\href{https://www.ml2r.de}{\bf ML2R}) which is funded by the Federal Ministry of Education and Research of Germany (grant no. 01|S18038C). The authors gratefully acknowledge this support.
	
	\bibliography{references}
	
	\clearpage
	\appendix


In this appendix, we provide proofs for the results of our main article and give additional details.
The numbering of theorems, definitions, etc. is identical to the main paper.
References to literature refer to the bibliography of the article above.

\section{(A) ~~~ The Filtration Kernel}
\label{apx:linearcombinationkernel}
    In this section, we prove the claims made in Sect.~\textit{Graph Filtration Kernels}. 
    \subsection{Details on the Filtration Kernel}
        \label{apx:filtration_kernel_details}
        In the description of the kernel was mentioned that the base kernels $k_f(G,G')$ are not well-defined on graphs $G,G'$ whenever feature $f$ appears only in $G$ but not $G'$ (or vice-versa). 
        This issue can formally be resolved by introducing an additional dummy entry in the histograms ${\phi}_f(G)$ and  ${\phi}_f(G')$ before normalization.
        We assign this new dummy entry the weight $1$ whenever $f$ does not appear in $G$ and $0$ otherwise.
        Recall that entries in the histograms are associated with values $\alpha_1,\ldots,\alpha_k$. 
        The value $\alpha_\delta$ associated with the dummy entry can be chosen at will.
        We note that the resulting Wasserstein distance $\W_{d^1}(\hat{\phi}_f(G),\hat{\phi}_f(G'))$ over such altered histograms is not necessarily meaningful.
        However, since $f$ does not appear in $G$ and thus $\norm{\phi_f(G)}=0$, the similarity $k_f(G,G')$ is effectively disregarded as the entire term $k_f(G,G') \norm{\phi_f(G)}_1 \norm{\phi_f(G')}_1$ becomes zero.
        In fact, therefore, $k_f(G,G')$ does not need to be computed to begin with. 
        It directly follows that only such features in $\F$ contribute to the similarity $K^\F_{\text{Filt}}(G,G')$ which appear in $G$ \emph{and} $G'$. 
        More formally:
        \begin{lemma}\label{lemma:finitestuff}
            For graphs $G,G'$, a set of features $\F$ and a feature $f' \not \in \F$, it holds that $K^{\F}_{Filt}(G,G') = K^{\F \cup f'}_{Filt}(G,G')$ if $f'$ does not appear in $G$ or $G'$. 
		\end{lemma}

    \subsection{Proof of Theorem \ref{thm:psd}}
        We first prove that the base kernel $k^\A_f$ is positive semi-definite.
        \begin{lemma}
            \label{lemma:base_kernel}
            The base kernel $k^\A_f$ is positive semi-definite.
		\end{lemma}
        \begin{proof}
            First, it needs to be shown that the Wasserstein distance $\W_{d}(\cdot,\cdot)$ using the 1-dimensional ground distance $d^1$ is \emph{conditionally negative definite} (CND). 
            This can be proven using results in \citet{DBLP:conf/nips/LeYFC19}  who showed that this even holds when the ground distance is defined by a tree metric $d_\mathcal{T}$.
            It is easy to see that $d^1$ is a special case of $d_\mathcal{T}$ as $d^1$ corresponds to a tree metric on a path. 
            According to \citet{schoenberg1938psd}, it holds that $k(x,y)=e^{-\gamma g(x,y)}$ is \emph{positive semi-definite} (PSD) for all $\gamma \in \bbR_+$ if $g$ is CND. 
            Thus, the base kernel $k^\A_f(G,G') = e^{-\gamma \W_{d^1}(\hat{\phi}^\A_f(G),\hat{\phi}^\A_f(G'))}$ is PSD. 
        \end{proof}
        Using this result we can now prove Theorem \ref{thm:psd}.
        \newtheorem*{repeatedpsd}{Theorem~\ref{thm:psd}}
        \begin{repeatedpsd}
            The filtration kernel $K^\F_{\text{Filt}}$ is positive semi-definite.
        \end{repeatedpsd}
        \begin{proof}
            Let $\tilde{k}_f(G,G') = \norm{\phi_f(G)}_1 \norm{\phi_f(G')}_1$ be the feature frequency product which is trivially PSD, then we can write 
            
                \[ K^\F_{\text{Filt}}(G,G') = \sum_{f \in \F} k_f(G,G') ~ \tilde{k}_f(G,G') ~.\] 
            
            Following, e.g., \citet{DBLP:books/lib/ScholkopfS02}, kernels are closed under (finite) addition and multiplication.
            By Lemma~\ref{lemma:finitestuff}, $K^{\F}_{Filt}(G,G')$ can be written as $K^{\F'}_{Filt}(G,G')$, where $\F'$ is the set of features that appear in both $G$ and $G'$ and which is assumed to be finite (cf. Sect. \textit{Background}).
            Furthermore, as $\F'$ is assumed to be finite and $K^{\F'}_{\text{Filt}}$ is not influenced by any features $f' \not \in \F'$, it follows that $K^{\F}_{\text{Filt}}$ is well defined and PSD on any finite set of graphs.
            In particular, adding a novel graph to an existing set does not change the kernel values of other graphs.
        \end{proof}
        
    \subsection{Proof of Proposition \ref{prop:filtration_to_histogram_kernel}}
        We prove that the filtration kernel is a generalization of the histogram kernel.
		\newtheorem*{repeatedprop}{Proposition~\ref{prop:filtration_to_histogram_kernel}}
		\begin{repeatedprop}
			For filtrations of length $k=1$, the filtration kernel $K^\F_{\text{Filt}}$ reduces to the histogram kernel $K^\F_H$.
		\end{repeatedprop}
		\begin{proof}
		    If there exists only a single filtration graph (i.e., $k=1$), then $k_f(G,G') = 1$ for all $f$ which appear in both $G$ and $G'$. 
		    (Recall that Sect.~A showed that features which do not appear in both graphs can be disregarded as they do not actively contribute to $K^\F_{\text{Filt}}$.)
		    We know that for $k=1$ it holds that $c(f(G)) = \norm{\phi_f(G)}_1$ where $c(f(G))$ is the number of occurences of features $f$ in $G$.
		    Thus, the filtration kernel reduces to $\sum_{f \in \F} c(f(G)) c(f(G')) = K^\F_H(G,G')$.
		\end{proof}

\section{(B) ~~~ The Product Variant of our Kernel}
\label{apx:productkernel}
In this section, we discuss an alternative variant of filtration kernels. 
This kernel is introduced only for purposes of demonstrating the universality of the filtration kernel concept and should be regarded as a theoretical contribution that is independent of the kernel introduced in Sect. \textit{Graph Filtration Kernels}. 

Filtration kernels can be defined as products over base kernels.
Analogously to the (linear combination based) kernel presented in our article, each base kernel is individually weighted to make up for the loss of information due to the normalization of histograms. 
This is realized by an RBF term measuring the similarity between original histogram masses.
\begin{definition}[Filtration (Product) Kernel]
	Given graphs $G,G'$, a filtration function $\A$, a set of features $\F$, and a parameter $\beta \in \mathbb{R}_{+}$, the \emph{filtration (Product) kernel} is given by 
	\begin{equation}
	\tilde{K}^{\F,\A}_{Filt}(G,G')=\prod_{f \in \F} k^\A_f(G,G') ~ e^{-\beta ~ (\lVert \phi^\A_f(G) \lVert_1 - \lVert \phi^\A_f(G') \lVert_1)^2 }
	\end{equation}
\end{definition}
\begin{theorem}
	\label{thm:product_kernel}
	The kernel $\tilde{K}^{\F,\A}_{Filt}$ is positive semi-definite.
\end{theorem}
\begin{proof} \textit{(Sketch.)}
	Lemma \ref{lemma:base_kernel} shows that the base kernel $k_f$ is positive semi-defininte (PSD). 
	Furthermore, the RBF term $e^{-\beta (x-y)^2 }$ with $x,y \in \bbR$ and $\beta \in \bbR_+$ is known to be PSD.
	The rest of the prove is to a large extend analog to that of Theorem \ref{thm:product_kernel} found in App.~A. 
	One difference, however, is that whereas the linear combination kernel $K^\F_{\text{Filt}}(G,G')$ depends only on features $f \in \F$ which appear in both graphs $G$ and $G'$, the product kernel variant $\tilde{K}^\F_{\text{Filt}}(G,G')$ is affected by features $f \in \F$ which appear in $G$ or $G'$. 
	Thus, the key difference between the two variants is the case where $G$ has a feature $f$ which $G'$ does not have. 
	Nonetheless, the number of features that appear in $G$ or $G'$ remains finite. 
	Furthermore, the kernel value is not changed by any feature $f'$ which does not appear in either graph since in this case $k_{f'}(G,G')=1$ and $e^{-\beta ~ (\lVert \phi_{f'}(G) \lVert_1 - \lVert \phi_{f'}(G') \lVert_1)^2} = 1$. 
	It follows that $\tilde{K}^{\F}_{Filt}(G,G')$ is PSD. 
\end{proof}

\begin{proposition}
	For filtrations $\A$ of length $k=1$, the kernel $\tilde{K}^{\F,\A}_{Filt}(G,G')$ reduces to the RBF kernel $e^{-\beta \norm{\varphi(G) - \varphi(G')}^2}$ where $\varphi(G)$ and $\varphi(G')$ denote the feature vectors of $G$, resp. $G'$ (c.f. Eq. \ref{eq:feature_vector}).
\end{proposition}
\begin{proof}
	In case $k=1$, there exists only a single entry in the filtration histograms. 
	By mass-normalizing histograms without adding a dummy element, all mass-normalized filtration histograms are trivially equivalent.
	It follows that $k^\A_f(G,G')=1$ for all $f \in \F$.
	Thus, the kernel reduces to the following:
	\begin{eqnarray*}
		\tilde{K}^{\F}_{Filt}(G,G')
		& = & \prod_{f \in \F} e^{-\beta ~ (\lVert \phi_f(G) \lVert_1 - \lVert \phi_f(G') \lVert_1)^2 } \\
		& = & \prod_{f \in \F} e^{-\beta ~ (c(f(G)) - c(f(G')))^2 } \\
		& = & e^{-\beta ~ \sum_{f \in \F} (c(f(G)) - c(f(G')))^2 } \\
		& = & e^{-\beta \norm{\varphi(G) - \varphi(G')}^2} \\
	\end{eqnarray*}
	
\end{proof}

\section{(C) ~~~ Complexity Analysis}
\label{apx:kernel:complexity}
We now discuss the complexity of the Weisfeiler-Lehman filtration kernel defined in Section~\textit{Graph Filtration Kernels}.
The proof for the product kernel $\tilde{K}^{\F,\A}_{Filt}$ defined above is analogous.
\newtheorem*{repeatedcomplex}{Theorem~\ref{thm:complexity}}
\begin{repeatedcomplex}
	The Weisfeiler-Lehman filtration kernel $K^{\F_{WL}}_{\text{Filt}}(G,G')$ on graphs $G,G'$ can be computed in time $O(hkm)$.
\end{repeatedcomplex}
\begin{proof}
	The ordinary Weisfeiler-Lehman subtree kernel $K_{WL}(G,G')$ can be computed in time $O(hm)$ where $h$ is the depth parameter and $m$ is the number of edges in $G$ and $G'$. 
	The Weisfeiler-Lehman filtration kernel needs to perform the WL feature extraction step a total of $k$ times; once for every filtration graph. 
	Thus, the number of WL labels over all filtrations is bounded in $O(hkm)$. 
	It remains to be shown that the base kernel calculations $k_f(G,G')$ over all mutual features $f \in \F'$ has a total complexity of $O(hkm)$. 
	Note that the Wasserstein distance using the ground distance $d^1$ has a closed form \citep{DBLP:conf/nips/LeYFC19}. 
	Thus, for two histograms $\hat{\phi}_f(G)$ and $\hat{\phi}_f(G')$ the distance $W_{d^1}(\hat{\phi}_f(G), \hat{\phi}_f(G'))$ can be calculated in time linear in the number of non-zero entries of both histograms. 
	Since for graphs $G$ and $G'$ the total number of non-zero histograms entries over all features $f \in \F'$ is bounded in $O(hkm)$, it follows that the set of all $k_f(G,G'), f \in \F'$ must be computable in $O(hkm)$ as well. 
	%
\end{proof}

\section{(D) ~~~ The Expressive Power of Weisfeiler-Lehman Filtration Sequences}
\label{apx:kernel:power}

    In this section we will provide proof that Weisfeiler-Lehman (WL) labels over filtration sequences can be strictly more powerful than WL labels alone. 
    Note that by Weisfeiler-Lehman labels, we mean \emph{1-dimensional} WL labels (consistently with our usage in the main paper) and that the vertices of the given graphs are allowed to be labeled with discrete labels.
    In particular:
    
    \newtheorem*{repeatedpower}{Theorem~\ref{thm:power}}
    \begin{repeatedpower}
    There exists a filtration function $\A$ such that $\phi^\A_f(G)=\phi^\A_f(G')$ for all $f \in \F_{WL}$ if and only if $G$ and $G'$ are isomorphic.
    \end{repeatedpower}
    
    The proof of this theorem is based on the fact that isomorphism between graphs and canonicalization of graphs are closely related concepts.
    A \emph{canonical form} for a class of graphs $\mathcal{G}$ is a function $c: \mathcal{G} \to \mathcal{G}$ that assigns each $G \in \mathcal{G}$ a unique representative $c(G)$ of its (equivalence) class of graphs that are isomorphic to $G$.
    Correspondingly, a \emph{canonical ordering} $o_{c}$ of $G \in \mathcal{G}$ assigns vertices of $G$ the index of their images in $c(G)$ under some isomorphism from $G$ to $c(G)$.
    Hence a check for isomorphism between $G$ and $G'$ can be done by simply comparing $c(G)$ and $c(G')$ for equality and $c(G)$ can be obtained from $G$ by permuting its vertices according to the canonical ordering $o_c$.
    
    One folklore example of an algorithm that computes a canonical form is to fix some total order on the adjacency matrices\footnote{
        For graphs without vertex labels, the standard definition of adjacency matrices as binary matrices can be used. 
        For simple labeled graphs, we can define an adjacency matrix in this context as the $n \times n$ matrix $A$ that -- for a given ordering $v_1, v_2, \ldots, v_n$ of vertices of $G$ -- contains the label of $v_i$ at $A_{i,i}$ and a one at $A_{i,j}$ and $A_{j,i}$ if $\{ v_i, v_j \} \in E(G)$.
    } 
    of graphs in $\mathcal{G}$ (e.g. a lexicographic order over the flattened upper triangle of adjacency matrices) and by mapping $G \in \mathcal{G}$ to the smallest permutation of its adjacency matrix with respect to the total order.
    This algorithm can be applied to all graphs and is NP-complete in general \citep{DBLP:conf/stoc/BabaiL83}.
    However, efficient algorithms are known for several graph classes, e.g., polynomial time algorithms for graphs of bounded-treewidth \citep{DBLP:journals/eccc/Wagner11}.
    
    Any canonical ordering can be used to define edge weights that allow to obtain graph filtrations $\A(G)$ and $\A(G')$ which allow to exactly solve the subgraph isomorphism problem using Weisfeiler-Lehman labels alone, as we will now show.
    To this end, we note that the Weisfeiler-Lehman isomorphism test can always distinguish two graphs with identical vertex count but different numbers of edges.
    
    \begin{lemma}\label{lem:d1wl}
    Let $G, H$ be two graphs with $\abs{V(G)} = \abs{V(G')}$ but $\abs{E(G)} \neq \abs{E(G')}$.
    Then the sets of iteration-1 Weisfeiler-Lehman labels of $G$ and $G'$ differ.
    \end{lemma}
    
    \begin{proof}
    Note that the iteration-1 Weisfeiler-Lehman labels correspond to vertex degrees in the unlabeled case.
    More generally, it holds that two vertices (even in the case that graphs are labeled) with different numbers of neighbors cannot have identical 1-dimensional iteration-1 Weisfeiler-Lehman labels.
    Now suppose for contradiction that $G$ and $G'$ have the same sets of 1-dimensional iteration-1 Weisfeiler-Lehman labels. 
    Using the above insight, this implies 
    \[  \sum_{v\in V(G)} \delta(v) =  \sum_{w\in V(G')} \delta(w) \ , \]
    where $\delta(v)$ is the degree of node $v$. But
    \[  \sum_{v\in V(G)} \delta(v) = 2 \abs{E(G)} \neq 2 \abs{E(G')} = \sum_{w\in V(G')} \delta(w)  \ ,\]
    which is a contradiction.
    \end{proof}

    Using Lemma~\ref{lem:d1wl}, we can now prove Theorem~\ref{thm:power}.
    
    \begin{proof}[Proof of Theorem~\ref{thm:power}]
    Let $\mathcal{G}$ be a class of graphs and let $c$ be a canonical form for $\mathcal{G}$.
    For $G \in \mathcal{G}$, we define an edge weight function $w: E(G) \to \mathbb{R}_{+}$ using the corresponding canonical ordering $o_c$ of $G$.
    To this end, let $A(c(G))$ be the flattened upper triangle of the adjacency matrix of $c(G)$.
    Then, for $e = \{v,w\} \in E(G)$, we set $w(e)$ to the index of $\{o_c(v), o_c(w)\}$ in $A(c(G))$.
    This weight function then maps each edge to an integer weight larger than zero.
    See Fig.~\ref{fig:canonical_filtration} for an example of this weight function.
    
    \begin{figure}[t]
		\centering
		\resizebox{0.5\textwidth}{!}{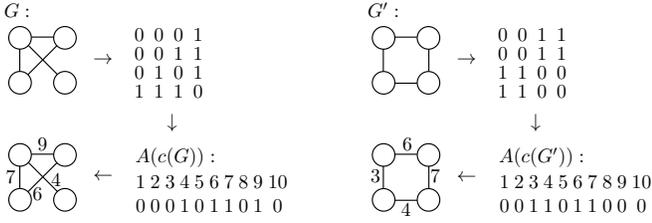}
		\caption{Graphs $G$ and $G'$ are assigned edge weights by a canonicalization process. The filtration graphs $G_9 = (V(G),\{e \in E(G): w(e) \geq 9\})$ and $G'_9 = (V(G'),\{e \in E(G'): w(e) \geq 9\})$ can be distinguished using WL labels as $G_9$ has an edge and $G'_9$ does not. Thus, graphs $G$ and $G'$ are not isomorphic.}
		\label{fig:canonical_filtration}
	\end{figure}
    
    If $G$ and $G'$ have different numbers of vertices, then the sets of WL labels of even $G$ and $G'$ have different cardinalities and the claim is trivial.
    Given two graphs $G, G'$ on $n$ vertices with edge weights defined as above, we can use the natural numbers as cutoff values, resulting in filtrations $\A(G), \A(G')$ of length ${n \choose 2} + n$.
    Now, if $G$ and $G'$ are isomorphic, then for each $i\in [{n \choose 2}]$ we have that the filtration graphs $G_i$ and $G'_i$ are isomorphic. 
    Hence the sets of WL labels of $G_i$ and $G'_i$ will be identical for any WL iteration.
    
    Now consider the case that $G$ and $G'$ are not isomorphic.
    Then it must be the case that their canonical forms are not equal, i.e., $c(G) \neq c(G')$.
    As a result, the adjacency matrices of $c(G)$ and $c(G')$ are not equal and there must be a \emph{largest} index $x$ where the entries of $A(c(G))$ and $A(c(G'))$ differ.
    Hence, $G_x$ and $G'_x$ are not isomorphic (recall, that $G_x = (V(G),\{e \in E(G):~ w(e) \geq x\})$), as they have different numbers of edges.
    As a result of Lemma~\ref{lem:d1wl}, the sets of iteration-1 WL labels of $G_x$ and $G'_x$ differ, which concludes the proof.
    \end{proof}
    
    Note that the construction in the proof of Theorem~\ref{thm:power} shows how to obtain a filtration function from a canonical form.
    While this result is only of theoretical interest 
    it shows that WL labels on top of filtrations $\A(G), \A(G')$ are strictly more powerful than WL labels on the original graphs $G, G'$.
    In particular, it shows that filtrations and WL labels are ``compatible'' in the sense that the information provided by a graph filtration may be used beneficially by the WL algorithm.
    Together with several conceivable efficiently computable filtrations (such as the example mentioned in Sect.~\textit{On the Expressive Power of Weisfeiler-Lehman Filtration Kernels} or the canonical form for bounded-treewidth graphs \cite{DBLP:journals/eccc/Wagner11}) this supports our claim that our approach can be beneficial in practice.
    
    We now prove the two corollaries on the implications of Theorem~\ref{thm:power} on the expressive power of our proposed kernels and on graph neural networks (GNNs).

    \newtheorem*{repeatedcorrkernels}{Corollary~\ref{corr:power:kernels}}
    \begin{repeatedcorrkernels}
        There exists a filtration function $\A$ such that the kernel $K^{\F_{WL},\A}_{\text{Filt}}$ is \emph{complete}.
    \end{repeatedcorrkernels}
    
    \begin{proof}
    We will use the edge weight functions and filtrations from the proof of Thm.~\ref{thm:power} and it remains to be shown that in this case 
    \[ K^{\F_{WL},\A}_{\text{Filt}} = \sum_{f \in \F_{WL}} k^\A_f(G,G') ~ \norm{\phi^\A_f(G)}_1 \norm{\phi^\A_f(G')}_1 \]
    is complete \citep{DBLP:conf/colt/GartnerFW03}, i.e., if its feature map $\varphi^{\F_{WL},\A}_{\text{Filt}}$ satisfies $\varphi^{\F_{WL},\A}_{\text{Filt}}(G)=\varphi^{\F_{WL},\A}_{\text{Filt}}(G') \Leftrightarrow G,G'$ are isomorphic.
    
    If $G$ and $G'$ are isomorphic then it trivially follows that $\varphi^{\F_{WL},\A}_{\text{Filt}}(G)=\varphi^{\F_{WL},\A}_{\text{Filt}}(G')$.
    Now suppose that $G$ and $G'$ are not isomorphic. 
    Then by Theorem~\ref{thm:power} there exists a WL feature $f' \in \F_{WL}$ such that $\phi^\A_f(G) \neq \phi^\A_f(G')$, as $G$ and $G'$ are not isomorphic.
    Now we note that the feature space of $K^{\F_{WL},\A}_{\text{Filt}}$ is the concatenation of the feature spaces of the kernels $\kappa_{f'}(G, G') = k^\A_{f'}(G,G') ~ \norm{\phi^\A_{f'}(G)}_1 \norm{\phi^\A_{f'}(G')}_1$ for all $f' \in \F_{WL}$.
    
    We will now show that the feature space embeddings of $G$ and $G'$ corresponding to $\kappa_f(\cdot, \cdot)$ do differ, which implies that the feature space embeddings of $G$ and $G'$ corresponding to $K^{\F_{WL},\A}_{\text{Filt}}$ differ, hence showing that $K^{\F_{WL},\A}_{\text{Filt}}$ is a complete graph kernel.
    Note that
    \begin{eqnarray*}
    \kappa_f(G, G') 
    & = & k^\A_f(G,G') ~ \norm{\phi^\A_f(G)}_1 \norm{\phi^\A_f(G')}_1 \\
    & = & \scalprod{\varphi^\A_f(G)}{\varphi^\A_f(G')} \norm{\phi^\A_f(G)}_1 \norm{\phi^\A_f(G')}_1 \\
    & = & \scalprod{\norm{\phi^\A_f(G)}_1 \varphi^\A_f(G)}{\norm{\phi^\A_f(G')}_1 \varphi^\A_f(G')} \ , \\
    \end{eqnarray*}
    where $\varphi^\A_f(H)$ is the feature vector of $H$ corresponding to the kernel $k^\A_f(\cdot, \cdot)$.
    Also recall, that $k^\A_f(G,G') = e^{-\gamma \W_{d^1}(\hat{\phi}^\A_f(G),\hat{\phi}^\A_f(G'))}$.
    We will distinguish two cases:
    
    \paragraph{Case 1} $\norm{\phi^\A_f(G)}_1 = \norm{\phi^\A_f(G')}_1$: \\
    Following from Thm.~\ref{thm:power} we still have $\hat{\phi}^\A_f(G) \neq \hat{\phi}^\A_f(G')$ for the normalized filtration histograms, as $G$ and $G'$ are not isomorphic.
    Thus $\W_d(\hat{\phi}^\A_f(G),\hat{\phi}^\A_f(G')) > 0$ and $\W_d(\hat{\phi}^\A_f(G),\hat{\phi}^\A_f(G)) = \W_d(\hat{\phi}^\A_f(G'),\hat{\phi}^\A_f(G')) = 0$, as $\W_d$ is a metric on normalized filtration histograms.
    Hence, the feature representations $\varphi^\A_f(G)$ and $\varphi^\A_f(G')$ cannot be identical, as $k_f^\A$ is a function and
    \[ k_f^\A(G, G) = k_f^\A(G', G') = 1 > k_f^\A(G, G') \ . \]
    
    \paragraph{Case 2} $\norm{\phi^\A_f(G)}_1 \neq \norm{\phi^\A_f(G')}_1$: \\
    To this end, we note that the feature representation $\varphi^\A_f(H)$ has unit lenght for all graphs $H$:
    \begin{eqnarray*}
    \norm{\varphi^\A_f(H)}^2_2 & = & \scalprod{\varphi^\A_f(H)}{\varphi^\A_f(H)} \\
    & = & k^\A_f(H,H) \\
    & = & e^{-\gamma \W_{d^1}(\hat{\phi}^\A_f(H),\hat{\phi}^\A_f(H))} \\
    & = & e^{-\gamma 0} = 1  \\
    \end{eqnarray*}
    by definition and due to the fact that $\W_{d^1}(\cdot, \cdot)$ is a metric.
    This implies $\norm{\varphi^\A_f(H)}_2 = 1$.
    Therefore,
    \begin{eqnarray*}
    \norm{\norm{\phi^\A_f(G)}_1 \varphi^\A_f(G)}_2
    & = & \norm{\phi^\A_f(G)}_1 \norm{ \varphi^\A_f(G)}_2 \\
    & = & \norm{\phi^\A_f(G)}_1 \\
    & \neq & \norm{\phi^\A_f(G')}_1 \\
    & = & \norm{\phi^\A_f(G')}_1 \norm{ \varphi^\A_f(G')}_2 \\
    & = & \norm{\norm{\phi^\A_f(G')}_1 \varphi^\A_f(G')}_2 \ . \\
    \end{eqnarray*}
    As two vectors with different lenghts cannot be identical, this concludes our proof.
    \end{proof}
    
    \newtheorem*{repeatedcorrgnns}{Corollary~\ref{corr:power:gnns}}
    \begin{repeatedcorrgnns}
    There exists a filtration function $\A$ and GNN \citep{DBLP:conf/iclr/XuHLJ19} $\mathcal{N}$ such that $\mathcal{N}$ can distinguish any two non-isomorphic graphs when provided with the filtration graphs corresponding to $\A$. 
    \end{repeatedcorrgnns}
    
    \begin{proof}
    This result directly follows from Theorem~3 of \citet{DBLP:conf/iclr/XuHLJ19}.
    Let $\mathcal{N}'$ be a GNN with injective aggregator and update functions, as well as an injective graph level readout function as stated by \citet{DBLP:conf/iclr/XuHLJ19}.
    We construct $\mathcal{N}$ by passing all filtration graphs given by filtration $\A$ of any graph $G$ through $\mathcal{N}'$ and adding another injective readout layer over the set of the resulting graph level readouts of the filtration graphs.
    By our Theorem~\ref{thm:power}, for two non-isomorphic graphs $G, G'$ there exist filtration graphs $G_x, G'_x$ that can be distinguished by their WL labels.
    Hence, following from \citet{DBLP:conf/iclr/XuHLJ19}, $\mathcal{N}'$ maps $G_x$ and $G'_x$ to different embeddings.
    Hence, an injective readout layer over these embeddings must map $G$ and $G'$ to different embeddings.
    \end{proof}
    
\section{(E) ~~~ Experimental Details}
\label{apx:datasets}
    \paragraph{Datasets}
        We conduct experiments on the well-established datasets DHFR, NCI1 and PTC-MR \citep[obtained from][]{DBLP:journals/corr/abs-2007-08663}) as representatives for the multitude of existing molecular benchmark datasets.  
   		We, furthermore, report results on several large network benchmark datasets.
   		IMDB-BINARY \citep[obtained from][]{DBLP:journals/corr/abs-2007-08663} consists of collaboration networks between actors, annotated against movie genres. 
   		The EGO datasets are novel real-world benchmark benchmarks extracted from the social networks Buzznet, Digg, Flickr and LiveJournal. Each of the EGO datasets consists of $50$ random ego network graphs from each of the $4$ social networks where graphs are annotated against the social network they were extracted from. The learning task is hence to assign each ego network to the network they were extracted from.
   		Here, ego networks are subgraphs induced by a central vertex's neighbors. The central vertex was subsequently removed. Graphs within each dataset were randomly chosen from the set of all ego networks but underlie size- and density-specific constraints to ensure that a simple count of nodes and edges is not sufficient for prediction tasks. The EGO-$x$ datasets contain increasingly larger and more dense ego networks with growing index $x$. 
   		
   		We selected only such datasets on which the baseline \BL kernel was outperformed by more sophisticated graph kernels.
        Detailed information about the structural properties of the benchmark dataset graphs are listed in Table \ref{tab:datasets}. 
        
        \begin{table}
           	\begin{center}
           		\resizebox{0.5\textwidth}{!}{  
            		\begin{tabular}{cccccc}
            			\hline
            			Dataset & \#graphs & \#classes & $\varnothing |V|$ & $\varnothing |E|$ & \#labels \\ 
            			\hline
            			DHFR & $756$ & $2$ & $42.4$ & $44.6$ & $9$ \\ 
            			NCI1 & $4110$ & $2$ & $29.9$ & $32.3$ & $37$ \\ 
            			PTC-MR & $344$ & $2$ & $14.3$ & $14.7$ & $18$ \\ 
            			IMDB-BINARY & $1000$ & $2$ & $19.8$ & $96.5$ & $1$ \\ 
            			EGO-1 & $200$ & $4$ & $139.0$ & $593.53$ & $1$ \\ 
            			EGO-2 & $200$ & $4$ & $178.6$ & $1444.9$ & $1$ \\ 
            			EGO-3 & $200$ & $4$ & $220.0$ & $2613.5$ & $1$ \\ 
            			EGO-4 & $200$ & $4$ & $259.8$ & $4135.8$ & $1$ \\ 
            		\end{tabular}
           		}%
           	\end{center}
           	\caption{Structural information of benchmark dataset graphs.}
           	\label{tab:datasets}
        \end{table}

 \paragraph{Experimental Setup}
     For all applied methods, we used the implementation as provided by \citep{JMLR:grakel} or that of the respective authors and ran grid searches using the following kernel specific parameters: 
   		For approaches employing the Weisfeiler-Lehman subtree features \WL, \WLOA, \WWL, \PWL as well as our kernel, we chose $h \in \{ 1,\ldots,5 \}$. 
   		In case of the graphlet sampling ($\GS$) kernel the parameters $\epsilon=0.1$, $\delta=0.1$ and $k \in \{3,4,5\}$ were applied.
   		If applicable to the individual kernel, we chose the RBF parameter $\gamma \in \{ 2^i: i \in \{ -12,-8,-5,-3,-1,0,1,3,5,8,12 \} \}$.
   		The SVM paratmeter $C$ was selected from $C \in \{ 2^i: i \in \{ -12,-8,-5,-3,-1,0,1,3,5,8,12 \} \}$.
   		In case of our filtration kernels, we set $k \in \{1,...,10\}$.

   	\paragraph{Runtime Analysis}
   		Fig. \ref{fig:runtime_k} provides the runtimes to compute the kernel matrix of our \FPWLdeg variant kernel. 
   		More specifically, we measured the runtime after the data has been loaded into memory. 
   		The datasets DHFR, NCI1, and PTC-MR contain only graphs with maximum node degree $3$, thus, limiting the range of the value $k$ to $3$ as well.
   		We note that these results were obtained using an unoptimized version of our kernel and are hence not necessarily representative for its potential real-world performance.
   		However, it is evident from the results that the runtime linearly increases with growing filtration length values $k$. 
   		This is consistent with our theoretical results on the runtime complexity of the kernel.
   		To put the above numbers into context, Tab. \ref{tab:runtime} provides the runtime measures for all tested competitor kernels. 
   		The \GS kernel was run using parameter $k=5$.
   		For all methods based on the Weisfeiler-Lehman process, we picked depth parameters $h=5$. 
   		
   		\begin{table}
   			\begin{center}
   				\begin{tabular}{ccccc}
   					\hline 
   					& DHFR & NCI1 & PTC-MR & IMDB-B. \\ 
   					\hline 
   					\GS 	&	77.17	&	434.32	&	38.64	&	135.89	\\
   					\SP 	&  	4.34	&	12.79	&	0.30	&	2.39	\\
   					\WL 	&	0.41	&	2.16	&	0.08	&	0.51	\\
   					\WLOA 	&  	12.39	&	1624.24	&	1.65	&	17.53	\\
   					\WWL 	&  	290.32	& 6457.34	& 32.49     &	297.06  \\
   					\PWL 	&  	4.27    &	27.78   &	0.68    &	12.55   \\
   					\CoreWL &  	1.28	&	7.92	&	0.21	&	6.65	\\
   					\hline
   					&	EGO-1 & EGO-2 & EGO-3 & EGO-4 \\
   					\hline
   					\GS 	&	21.62	&	23.38	&	24.44	&	26.58	\\
   					\SP 	&	15.96	&	33.27	&	60.01	&	97.18	\\
   					\WL 	&	0.50	&	0.92	&	1.47	&	2.16	\\
   					\WLOA 	&	4.16	&	6.14	&	7.93	&	9.85	\\
   					\WWL 	& 103.46    &	173.11  &	260.58  &	353.57	\\
   					\PWL 	&	18.48   &	55.65   &	119.40  &	212.05	\\
   					\CoreWL	&	5.08	&	17.18	&	40.26	&	78.98	\\
   				\end{tabular}
   			\end{center}
   			\caption{Runtime measures of competitor kernels (in seconds).}
   			\label{tab:runtime}
   		\end{table}
   	
   		\begin{figure}
   			\centering
  				\resizebox{0.45\textwidth}{!}{\begin{tikzpicture}
\begin{axis}[
	width=12cm,
	height=8cm,
	xtick={1,...,10},
	ymin=0, ymax=130,
	grid=major,
	grid style={line width=.5pt, draw=gray!20},
    legend style={at={(0.5,1)}, anchor=north, legend columns=4, /tikz/every even column/.append style={column sep=0.25cm}},
	ylabel={Runtime (in seconds)},
	ylabel near ticks,
	xlabel={$k$},
]

\addplot[color=baselinecolor, line width=0.25mm, mark=square*] 
table [y error=error] {
	x   y		error    
	1	3.7135956287384
	2	5.74284839630127
	3	6.26330089569092
};

\addplot[color=CoreWLcolor, line width=0.25mm, mark=halfsquare*] 
table [y error=error] {
	x   y		error    
	1	31.3129968643188
	2	47.8910040855408
	3	63.8491287231445
};

\addplot[color=SPcolor, line width=0.25mm, mark=diamond*] 
table [y error=error] {
	x   y		error    
	1	0.781185865402222
	2	0.755894660949707
	3	0.911377668380737
};

\addplot[color=WLcolor, line width=0.25mm, mark=halfcircle*] 
table [y error=error] {
	x   y		error    
	1	2.4805588722229
	2	3.96355509757996
	3	5.2160177230835
	4	6.66282796859741
	5	7.99582171440125
	6	8.9306800365448
	7	10.5883951187134
	8	12.0089676380157
	9	12.5213325023651
	10	13.8884162902832
};

\addplot[color=WLOAcolor, line width=0.25mm, mark=*] 
table [y error=error] {
	x   y		error    
	1	3.70302820205688
	2	5.38424134254456
	3	7.36605620384216
	4	9.21760630607605
	5	10.6025629043579
	6	12.4896759986877
	7	14.8089907169342
	8	16.6008396148682
	9	17.8341019153595
	10	19.4845380783081
};

\addplot[color=WWLcolor, line width=0.25mm, mark=x] 
table [y error=error] {
	x   y		error    
	1	6.39702296257019
	2	10.7157509326935
	3	14.5423889160156
	4	18.549079656601
	5	22.6750347614288
	6	26.6182923316956
	7	29.3715844154358
	8	33.3245456218719
	9	38.1719641685486
	10	40.2717454433441
};

\addplot[color=FPWLdegcolor, line width=0.25mm, mark=triangle*] 
table [y error=error] {
	x   y		error    
	1	9.86773920059204
	2	17.3160629272461
	3	23.1892046928406
	4	29.5076622962952
	5	36.2762224674225
	6	42.8592729568481
	7	49.1328547000885
	8	58.1565685272217
	9	64.2216606140137
	10	68.2102129459381
};

\addplot[color=FPWLrwcolor, line width=0.25mm, mark=pentagon*] 
table [y error=error] {
	x   y		error    
	1	14.8502621650696
	2	25.7161281108856
	3	34.9056570529938
	4	45.8966212272644
	5	54.0745258331299
	6	64.4354698657989
	7	74.6624441146851
	8	84.1715521812439
	9	95.5188806056976
	10	106.927093029022
};

\legend{DHFR, NCI1, PTC-MR, IMDB-B., EGO-1, EGO-2, EGO-3, EGO-4}
\end{axis}
\end{tikzpicture}}
  				\captionof{figure}{Runtime measures of the \FPWLdeg variant for different filtration length values $k$.}
  				\label{fig:runtime_k}
  			\end{figure}
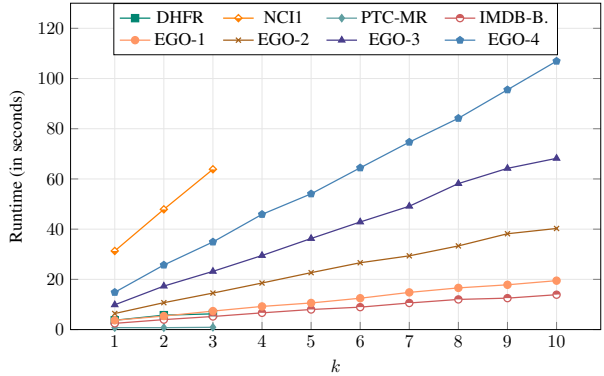

   	\paragraph{Computing Power}
  		 All experiments were performed on an AMD $3900$X processor (12 cores) with $64$GB of memory. 
  		 We implemented our filtration kernel in Python $3.7$. 

\end{document}